\documentclass{article}

\usepackage{microtype}
\usepackage{graphicx}
\usepackage{subcaption}
\usepackage{booktabs}
\usepackage{hyperref}
\usepackage{enumitem}

\usepackage[preprint]{icml2026}

\usepackage{amsmath}
\usepackage{amssymb}
\usepackage{mathtools}
\usepackage{amsthm}
\usepackage[capitalize,noabbrev]{cleveref}

\theoremstyle{plain}
\newtheorem{theorem}{Theorem}[section]
\newtheorem{proposition}[theorem]{Proposition}
\newtheorem{lemma}[theorem]{Lemma}
\newtheorem{corollary}[theorem]{Corollary}
\theoremstyle{definition}
\newtheorem{definition}[theorem]{Definition}

\theoremstyle{remark}
\newtheorem{remark}[theorem]{Remark}

\newcommand{\lse}{\mathrm{LSE}}
\newcommand{\softmax}{\mathrm{softmax}}
\newcommand{\R}{\mathbb{R}}
\newcommand{\eps}{\varepsilon}

\icmltitlerunning{UltraLIF: Differentiable SNNs via Ultradiscretization}

\begin{document}

\twocolumn[
\icmltitle{UltraLIF: Fully Differentiable Spiking Neural Networks \\
           via Ultradiscretization and Max-Plus Algebra}


\begin{icmlauthorlist}
\icmlauthor{Jose Marie Antonio Mi\~{n}oza}{cair}
\end{icmlauthorlist}

\icmlaffiliation{cair}{Center for AI Research PH}

\icmlcorrespondingauthor{Jose Marie Antonio Mi\~{n}oza}{jminoza@upd.edu.ph}

\icmlkeywords{Spiking Neural Networks, Ultradiscretization, Max-Plus Algebra, Neuromorphic Computing}

\vskip 0.3in
]

\printAffiliationsAndNotice{}

\begin{abstract}
Spiking Neural Networks (SNNs) offer energy-efficient, biologically plausible computation but suffer from non-differentiable spike generation, necessitating reliance on heuristic surrogate gradients. This paper introduces \textbf{UltraLIF}, a principled framework that replaces surrogate gradients with \emph{ultradiscretization}, a mathematical formalism from tropical geometry providing continuous relaxations of discrete dynamics. The central insight is that the max-plus semiring underlying ultradiscretization naturally models neural threshold dynamics: the log-sum-exp function serves as a differentiable soft-maximum that converges to hard thresholding as a learnable temperature parameter $\eps \to 0$. Two neuron models are derived from distinct dynamical systems: UltraLIF from the LIF ordinary differential equation (temporal dynamics) and UltraDLIF from the diffusion equation modeling gap junction coupling across neuronal populations (spatial dynamics). Both yield fully differentiable SNNs trainable via standard backpropagation with no forward-backward mismatch. Theoretical analysis establishes pointwise convergence to classical LIF dynamics with quantitative error bounds and bounded non-vanishing gradients. Experiments on six benchmarks spanning static images, neuromorphic vision, and audio demonstrate improvements over surrogate gradient baselines, with gains most pronounced in single-timestep ($T{=}1$) settings on neuromorphic and temporal datasets. An optional sparsity penalty enables significant energy reduction while maintaining competitive accuracy.
\end{abstract}

\section{Introduction}
\label{sec:intro}

Spiking Neural Networks (SNNs) represent a promising paradigm for energy-efficient machine learning, with significant potential for neuromorphic hardware deployment \citep{maass1997networks, roy2019towards}. In contrast to artificial neural networks (ANNs) communicating via continuous activations, SNNs process information through discrete spike events, emulating biological neural computation. This event-driven nature enables substantial energy savings; Intel's Loihi chip demonstrates up to $1000\times$ energy reduction compared to GPUs on certain tasks \citep{davies2018loihi}.

However, SNN training remains challenging due to the \emph{non-differentiability} of spike generation. The standard Leaky Integrate-and-Fire (LIF) neuron follows the dynamics:
\begin{align}
v^{(t+1)} &= \tau v^{(t)} + I^{(t)} \label{eq:lif_voltage}\\
s^{(t+1)} &= H(v^{(t+1)} - \theta) \label{eq:lif_spike}
\end{align}
where the Heaviside step function $H(\cdot)$ has gradient zero almost everywhere. The dominant approach employs \emph{surrogate gradients}, replacing the true gradient with a smooth approximation during backpropagation \citep{neftci2019surrogate, zenke2021remarkable}. While empirically effective, surrogate gradients introduce a fundamental mismatch between forward (discrete) and backward (continuous) passes (Figure~\ref{fig:spike_comparison}a), with limited theoretical understanding of convergence properties \citep{li2021differentiable, gygax2025elucidating}.

\begin{figure*}[t]
\centering
\includegraphics[width=\textwidth]{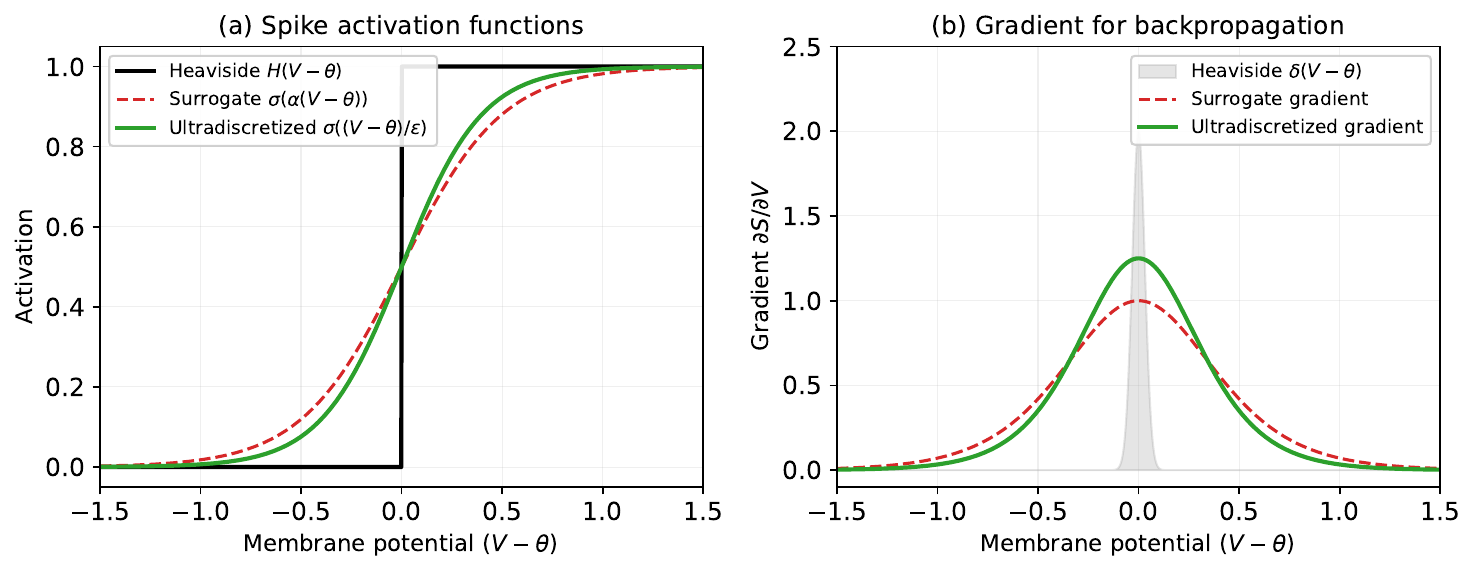}
\caption{\textbf{(a)} Spike activation functions: Heaviside (hard threshold), surrogate gradient (smooth approximation), and ultradiscretized (principled soft relaxation). \textbf{(b)} Gradients: Heaviside has zero gradient almost everywhere (delta function at threshold); surrogate and ultradiscretized provide smooth gradients, but only ultradiscretized maintains forward-backward consistency.}
\label{fig:spike_comparison}
\end{figure*}

This paper proposes \textbf{UltraLIF}, a theoretically grounded alternative based on \emph{ultradiscretization}, a limiting procedure from tropical geometry transforming continuous dynamical systems into discrete max-plus systems while preserving structural properties \citep{tokihiro1996ultra, grammaticos2004discrete}. The key contributions are:

\begin{enumerate}
    \item \textbf{Principled differentiability}: The $\lse$ function provides a natural soft relaxation of the max operation underlying spike generation, with explicit convergence bounds as $\eps \to 0$ (Lemma~\ref{lem:lse_convergence}).

    \item \textbf{Forward-backward consistency}: Unlike surrogate methods, UltraLIF employs identical dynamics in forward and backward passes, eliminating gradient mismatch (Remark~\ref{thm:consistency}).

    \item \textbf{Bounded gradients}: For any $\eps > 0$, gradients remain bounded and non-vanishing, enabling stable optimization (Proposition~\ref{prop:gradients}).

    \item \textbf{Consistent low-timestep improvements}: On six benchmarks spanning static, neuromorphic, and audio modalities, ultradiscretized models improve over surrogate gradient baselines at $T{=}1$, with the largest gains on temporal and event-driven data (+11.22\% SHD, +7.96\% DVS-Gesture, +3.91\% N-MNIST).
\end{enumerate}

\section{Related Work}
\label{sec:related}

\textbf{Surrogate Gradient Methods.} The dominant paradigm for direct SNN training replaces non-differentiable spike gradients with smooth surrogates \citep{neftci2019surrogate}. Common choices include piecewise linear \citep{bellec2018long}, sigmoid \citep{zenke2018superspike}, and arctangent \citep{fang2021incorporating} functions. Recent work introduces learnable surrogate parameters \citep{lian2023learnable} and adaptive shapes \citep{li2021differentiable}. Despite empirical success, the forward-backward mismatch remains theoretically problematic. \citet{gygax2025elucidating} provide partial justification via stochastic neurons, showing surrogate gradients match escape noise derivatives in expectation.

\textbf{Spike Timing Approaches.} SpikeProp \citep{bohte2002error} and variants \citep{mostafa2018supervised, kheradpisheh2020temporal} compute exact gradients with respect to spike times. These methods require careful initialization and struggle with silent neurons. Recent work on exact smooth gradients through spike timing \citep{goltz2021fast} addresses some limitations but remains computationally intensive.

\textbf{ANN-to-SNN Conversion.} An alternative approach trains conventional ANNs then converts to SNNs \citep{cao2015spiking, rueckauer2017conversion, bu2022optimal}. While avoiding direct SNN training, conversion methods typically require many timesteps to achieve ANN accuracy and sacrifice temporal dynamics.

\textbf{Tropical Geometry and Neural Networks.} Tropical geometry studies algebraic structures where addition becomes max (or min) and multiplication becomes addition \citep{maclagan2015introduction}. \citet{zhang2018tropical} establish that ReLU networks compute tropical rational functions. Recent work extends this to graph neural networks \citep{pham2024gnn} and neural network compression \citep{fotopoulos2024tropnnc}. The connection to spiking networks via ultradiscretization appears novel.

\textbf{Ultradiscretization.} Originating in integrable systems \citep{tokihiro1996ultra}, ultradiscretization transforms difference equations into cellular automata preserving solution structure. Applications include soliton systems \citep{takahashi1990soliton} and limit cycle analysis \citep{yamazaki2023emergence, yamazaki2024generalization}. Application to neural networks has not been previously explored.

\section{Preliminaries}
\label{sec:prelim}

\subsection{The Max-Plus Semiring}
\label{sec:maxplus}

\begin{definition}[Max-Plus Semiring]
The \emph{max-plus semiring} is the algebraic structure $(\R \cup \{-\infty\}, \oplus, \odot)$ where:
\begin{itemize}
    \item $a \oplus b := \max(a, b)$ (tropical addition)
    \item $a \odot b := a + b$ (tropical multiplication)
\end{itemize}
with additive identity $-\infty$ and multiplicative identity $0$.
\end{definition}

This semiring underlies tropical geometry and provides the limit structure for ultradiscretization.

\subsection{Log-Sum-Exp as Soft Maximum}
\label{sec:lse}

The log-sum-exp function with temperature $\eps > 0$ is defined as:
\begin{equation}
\lse_\eps(\mathbf{x}) := \eps \log\left(\sum_{i=1}^n e^{x_i/\eps}\right)
\label{eq:lse_def}
\end{equation}

The following lemma establishes its role as a smooth approximation to the maximum.

\begin{lemma}[LSE Convergence]
\label{lem:lse_convergence}
Let $\mathbf{x} = (x_1, \ldots, x_n) \in \R^n$ and $M := \max_i x_i$. Then $M \leq \lse_\eps(\mathbf{x}) \leq M + \eps \log n$, hence $\lim_{\eps \to 0^+} \lse_\eps(\mathbf{x}) = M$. Moreover, $\nabla \lse_\eps(\mathbf{x}) = \softmax(\mathbf{x}/\eps) \in (0,1)^n$.
\end{lemma}

\begin{proof}
For the lower bound, the sum includes $e^{M/\eps}$, so $\lse_\eps(\mathbf{x}) \geq \eps \log(e^{M/\eps}) = M$. For the upper bound, since $x_i \leq M$ for all $i$, $\lse_\eps(\mathbf{x}) \leq \eps \log(n \cdot e^{M/\eps}) = M + \eps \log n$. The limit follows by the squeeze theorem. The gradient formula follows from direct differentiation: $\frac{\partial \lse_\eps}{\partial x_j} = \frac{e^{x_j/\eps}}{\sum_i e^{x_i/\eps}} = \softmax(\mathbf{x}/\eps)_j$.
\end{proof}

\begin{remark}
When the maximum is unique with gap $\delta := M - \max_{i: x_i \neq M} x_i > 0$, the error decays exponentially: $\lse_\eps(\mathbf{x}) - M = O(\eps e^{-\delta/\eps})$.
\end{remark}

\section{Method: Ultradiscretized Spiking Neurons}
\label{sec:method}

The key innovation of this work is applying \emph{ultradiscretization}, a limiting procedure from tropical geometry, to derive differentiable spiking neurons. This section shows that ultradiscretization can be applied to different neural dynamics, yielding distinct models for temporal and spatial computations.

\subsection{Ultradiscretization Framework}
\label{sec:ultradiscret}

Ultradiscretization transforms continuous dynamical systems into max-plus (tropical) systems while preserving structural properties \citep{tokihiro1996ultra}. The procedure operates via the substitution $x = e^{X/\eps}$ followed by the limit $\eps \to 0^+$:
\begin{align}
x + y = e^{X/\eps} + e^{Y/\eps} &\to e^{\max(X,Y)/\eps} \quad \text{(addition $\to$ max)} \label{eq:ultra_add}\\
x \cdot y = e^{X/\eps} \cdot e^{Y/\eps} &= e^{(X+Y)/\eps} \quad \text{(multiplication $\to$ addition)} \label{eq:ultra_mult}
\end{align}

For finite $\eps > 0$, the log-sum-exp function $\lse_\eps$ (Eq.~\ref{eq:lse_def}) provides a differentiable soft relaxation of the max operation. This is the foundation for all ultradiscretized neurons presented here.

\subsubsection{Temporal Dynamics: UltraLIF}
\label{sec:ultralif_temporal}

UltraLIF is derived from the standard single-neuron LIF ordinary differential equation. The membrane potential evolves according to:
\begin{equation}
\tau_m \frac{dv}{dt} = -(v - v_\text{rest}) + R \cdot I(t)
\label{eq:lif_ode}
\end{equation}
where $\tau_m$ is the membrane time constant, $v_\text{rest}$ the resting potential, $R$ the membrane resistance, and $I(t)$ the input current.

Applying forward Euler discretization with timestep $\Delta t$ and setting $v_\text{rest} = 0$ yields:
\begin{equation}
v^{(t+1)} = \underbrace{\left(1 - \frac{\Delta t}{\tau_m}\right)}_{=:\tau_0} v^{(t)} + I^{(t)}
\label{eq:lif_discrete}
\end{equation}
where the leak factor $\tau_0 \in (0,1)$ controls temporal decay.

The ultradiscretization transform \citep{tokihiro1996ultra} proceeds by substituting $v = e^{V/\eps}$, $I = e^{J/\eps}$, and parameterizing the leak as $\tau = e^{T/\eps}$ where $T = \log\tau_0 < 0$:
\begin{equation}
e^{V^{(t+1)}/\eps} = e^{T/\eps} \cdot e^{V^{(t)}/\eps} + e^{J^{(t)}/\eps} = e^{(V^{(t)} + T)/\eps} + e^{J^{(t)}/\eps}
\end{equation}
Taking $\eps \cdot \log$ of both sides and the limit $\eps \to 0^+$ recovers the max-plus dynamics:
\begin{equation}
V^{(t+1)} = \max\left(V^{(t)} + T, \, J^{(t)}\right) = \max\left(V^{(t)} + \log\tau_0, \, J^{(t)}\right)
\label{eq:ultralif_limit}
\end{equation}

For differentiable training, the hard maximum is relaxed to the log-sum-exp for finite $\eps > 0$:
\begin{equation}
V_\eps^{(t+1)} = \lse_\eps\left(V^{(t)} + \log\tau_0, \, I^{(t)}\right)
\label{eq:ultralif_soft}
\end{equation}
This \textbf{2-term LSE} captures temporal membrane integration with learnable leak. Note that the $\eps$-parameterized leak $\tau = e^{T/\eps}$ becomes stronger as $\eps \to 0$, yielding sharper temporal dynamics in the tropical limit.

\subsubsection{Spatial Dynamics: UltraDLIF}
\label{sec:ultradlif_spatial}

An analogous derivation applies ultradiscretization to spatial dynamics, capturing lateral interactions across a neuronal population. Consider a simplified diffusive coupling where membrane potentials spread locally:
\begin{equation}
\frac{\partial v}{\partial t} = D \nabla^2 v
\label{eq:nf_pde}
\end{equation}
where $D > 0$ is the diffusion coefficient. This models gap junction (electrical synapse) coupling, where ionic currents flow directly between neurons enabling voltage spread \citep{connors2004electrical, spek2020neural}. The diffusion equation provides a first-order approximation of such lateral interactions.

Discretizing the Laplacian via finite differences $\nabla^2 v \approx (v_{i-1} - 2v_i + v_{i+1})/\Delta x^2$ and applying forward Euler in time yields:
\begin{equation}
v_i^{(t+1)} = v_i^{(t)} + \frac{D \Delta t}{\Delta x^2}(v_{i-1}^{(t)} - 2v_i^{(t)} + v_{i+1}^{(t)})
\label{eq:nf_discrete}
\end{equation}

At the balanced diffusion regime where $D\Delta t/\Delta x^2 = 1/3$, this simplifies to uniform spatial averaging where each neuron and its neighbors contribute equally:
\begin{equation}
v_i^{(t+1)} = \frac{1}{3}v_{i-1}^{(t)} + \frac{1}{3}v_i^{(t)} + \frac{1}{3}v_{i+1}^{(t)}
\label{eq:nf_uniform}
\end{equation}
This choice of $1/3$ lies within the von Neumann stability bound for explicit finite difference schemes applied to the 1D diffusion equation, which requires $D\Delta t/\Delta x^2 \leq 1/2$ for numerical stability. The value $1/3$ ensures stability while providing symmetric treatment of a neuron and its immediate neighbors, a natural balance for lateral coupling.

\textbf{Remark on subtraction.} Standard ultradiscretization cannot handle subtraction directly, as there is no tropical analog of $x - y$ in the max-plus semiring \citep{ochiai2005inversible}. This limitation is circumvented by selecting $\alpha = 1/3$: expanding Eq.~\eqref{eq:nf_discrete} gives coefficients $\alpha$, $(1{-}2\alpha)$, $\alpha$ for the three terms, and at $\alpha = 1/3$ all become $1/3 > 0$, eliminating subtraction entirely (Eq.~\eqref{eq:nf_uniform}). For more general diffusion regimes, inversible max-plus algebras extend the framework to handle subtraction via $x - y \to \max(X, Y + \eta)$, where $\eta$ is an inverse element \citep{ochiai2005inversible}.

Applying the ultradiscretization transform with $v = e^{V/\eps}$ and taking $\eps \to 0^+$:
\begin{equation}
V_i^{(t+1)} = \max\left(V_{i-1}^{(t)}, V_i^{(t)}, V_{i+1}^{(t)}\right)
\label{eq:ultradlif_limit}
\end{equation}
This limit corresponds to \emph{morphological dilation}, a max-pooling operation over the spatial neighborhood.

Relaxing the hard maximum to the log-sum-exp for finite $\eps > 0$ gives:
\begin{equation}
V_{i,\eps}^{(t+1)} = \lse_\eps\left(V_{i-1}^{(t)}, V_i^{(t)}, V_{i+1}^{(t)}\right)
\label{eq:ultradlif_soft}
\end{equation}
This \textbf{3-term LSE} captures lateral spatial smoothing across neurons. External input $I_i^{(t)}$ is added separately (Eq.~\ref{eq:ultradlif_voltage}), following the standard neural field convention where diffusion handles lateral coupling and an additive term represents external drive.

\subsubsection{Comparison of Derivations}

\begin{table}[h]
\centering
\begin{small}
\begin{tabular}{lcc}
\toprule
& \textbf{UltraLIF} & \textbf{UltraDLIF} \\
\midrule
Source & LIF ODE & Diffusion PDE \\
Equation & $dv/dt = -v/\tau_m + I$ & $\partial v/\partial t = D\nabla^2 v$ \\
LSE terms & 2 (temporal) & 3 (spatial) \\
Soft form & $\lse(V{+}\log\tau_0, I)$ & $\lse(V_{-1}, V_0, V_{+1})$ \\
Models & Membrane decay & Lateral diffusion \\
\bottomrule
\end{tabular}
\end{small}
\caption{Ultradiscretization applied to temporal and spatial dynamics.}
\label{tab:derivations}
\end{table}

Both models share the same theoretical foundation (ultradiscretization, LSE soft relaxation) but capture different biological phenomena. UltraLIF models single-neuron temporal dynamics; UltraDLIF models population-level spatial interactions.

\subsection{Neuron Models}
\label{sec:neuron_models}

Building on the ultradiscretization framework, complete neuron models are defined incorporating spike generation and reset mechanisms. Both variants share the same spike and reset logic, differing only in membrane dynamics.

\begin{definition}[UltraLIF Neuron (Temporal)]
\label{def:ultralif}
For temperature $\eps > 0$, leak factor $\tau_0 \in (0,1)$, and threshold $\theta > 0$:
\begin{align}
\tilde{V}_\eps^{(t+1)} &= \lse_\eps\left(V_\eps^{(t)} + \log\tau_0, \, I^{(t)}\right) \label{eq:ultralif_voltage}\\
s_\eps^{(t+1)} &= \sigma\left(\frac{\tilde{V}_\eps^{(t+1)} - \theta}{\eps}\right) \label{eq:ultralif_spike}\\
V_\eps^{(t+1)} &= \tilde{V}_\eps^{(t+1)} \cdot (1 - s_\eps^{(t+1)}) + V_\text{reset} \cdot s_\eps^{(t+1)} \label{eq:ultralif_reset}
\end{align}
where $\sigma(z) = (1 + e^{-z})^{-1}$ is the logistic sigmoid and $V_\text{reset} = 0$.
\end{definition}

\begin{definition}[UltraDLIF Neuron (Spatial)]
\label{def:ultradlif}
For temperature $\eps > 0$, threshold $\theta > 0$, and neuron index $i$:
\begin{align}
\tilde{V}_{i,\eps}^{(t+1)} &= \lse_\eps\left(V_{i-1,\eps}^{(t)}, V_{i,\eps}^{(t)}, V_{i+1,\eps}^{(t)}\right) + I_i^{(t)} \label{eq:ultradlif_voltage}\\
s_{i,\eps}^{(t+1)} &= \sigma\left(\frac{\tilde{V}_{i,\eps}^{(t+1)} - \theta}{\eps}\right) \label{eq:ultradlif_spike}\\
V_{i,\eps}^{(t+1)} &= \tilde{V}_{i,\eps}^{(t+1)} \cdot (1 - s_{i,\eps}^{(t+1)}) + V_\text{reset} \cdot s_{i,\eps}^{(t+1)} \label{eq:ultradlif_reset}
\end{align}
The LSE operates over the spatial neighborhood (circular boundary conditions).
\end{definition}

\textbf{Soft Spike Mechanism.} The soft spike $s_\eps \in (0,1)$ interpolates between no-spike ($s_\eps \approx 0$) and spike ($s_\eps \approx 1$). This differs fundamentally from surrogate gradient methods:
\begin{itemize}
    \item \textbf{Surrogate}: Forward uses hard $H(V - \theta)$; backward uses smooth $g'(V)$
    \item \textbf{Ultradiscretized}: Both forward and backward use same smooth $\sigma((V-\theta)/\eps)$
\end{itemize}


\textbf{On Soft vs.\ Binary Spikes.} A natural concern is that $s_\eps \in (0,1)$ does not represent ``true'' binary spikes. However, this deviation from binary behavior is not problematic in practice. The soft spike is used during training for gradient computation, while at inference one can employ small $\eps$ or hard thresholding ($s = H(V - \theta)$) for neuromorphic deployment, following a standard training-inference separation analogous to dropout or batch normalization. Furthermore, the output layer uses mean spike rate $\hat{y} = \frac{1}{T}\sum_t s^{(t)}$, which is inherently robust to soft versus hard individual spikes since the classification decision depends on aggregate activity rather than precise spike values. The sparsity penalty $\lambda \cdot \bar{s}$ encourages sparse soft activations that correspond to sparse hard spikes in the limit, ensuring energy-efficient inference. Most importantly, Proposition~\ref{prop:convergence} guarantees that $s_\eps \to H(V - \theta)$ as $\eps \to 0$, providing a principled path from soft training dynamics to binary inference behavior.

\textbf{Learnable Parameters.} The temperature $\eps$ is made learnable via $\eps = \exp(\log\eps_\text{param})$, initialized to $\eps_0 = 1.0$. For UltraLIF, the leak factor $\tau_0$ can also be learned (UltraPLIF variant). For UltraDLIF, similarly UltraDPLIF learns $\tau_0$ for an optional temporal component. During training, the network discovers optimal soft-to-hard trade-offs, implementing automatic curriculum learning.

\subsection{Network Architecture}
\label{sec:architecture}

A feedforward SNN with $L$ layers of UltraLIF neurons is constructed as:
\begin{align}
I_l^{(t)} &= W_l \cdot s_{l-1}^{(t)} + b_l \\
V_l^{(t+1)} &= \lse_\eps\left(V_l^{(t)} + \log\tau_0, \, I_l^{(t)}\right)(1 - s_l^{(t)}) + V_\text{reset} \cdot s_l^{(t)} \\
s_l^{(t+1)} &= \sigma\left((V_l^{(t+1)} - \theta) / \eps\right)
\end{align}
where $l \in \{1, \ldots, L\}$ indexes layers. The output layer employs spike rate coding:
\begin{equation}
\hat{y} = \frac{1}{T}\sum_{t=1}^T s_L^{(t)}
\end{equation}

\section{Theoretical Analysis}
\label{sec:theory}

\subsection{Convergence to LIF Dynamics}

\begin{lemma}[Sigmoid Convergence]
\label{lem:sigmoid}
Let $\sigma_\eps(x) := \sigma(x/\eps)$. For $x \neq 0$, $\lim_{\eps \to 0^+} \sigma_\eps(x) = H(x)$ with exponential convergence rate $|\sigma_\eps(x) - H(x)| \leq e^{-|x|/\eps}$.
\end{lemma}

\begin{proof}
For $x > 0$: $\sigma_\eps(x) = (1 + e^{-x/\eps})^{-1} \to 1$ as $\eps \to 0^+$.

For $x < 0$: $\sigma_\eps(x) = e^{x/\eps}/(e^{x/\eps} + 1) \to 0$ as $\eps \to 0^+$.

The error bound follows from $|1 - \sigma_\eps(x)| = e^{-x/\eps}/(1 + e^{-x/\eps}) \leq e^{-x/\eps} = e^{-|x|/\eps}$ for $x > 0$. For $x < 0$: $|\sigma_\eps(x)| = 1/(1 + e^{-x/\eps}) \leq e^{x/\eps} = e^{-|x|/\eps}$.
\end{proof}

\begin{proposition}[Convergence to LIF]
\label{prop:convergence}
Let $\{V_\eps(t)\}$ denote the UltraLIF trajectory with temperature $\eps > 0$, and $\{v(t), s(t)\}$ the standard LIF trajectory with $V_\mathrm{reset} = 0$. Assume (A1)~bounded inputs $|I(t)| \leq I_{\max}$ and (A2)~threshold margin $\delta_t := |v(t) - \theta| > 0$. Then $\lim_{\eps \to 0^+} V_\eps(t) = v(t)$ and $\lim_{\eps \to 0^+} s_\eps(t) = s(t)$ for each $t$, with $|V_\eps(t) - v(t)| \leq t \cdot \eps \log 2$ and $|s_\eps(t) - s(t)| \leq e^{-\delta_t/\eps}$. The linear error growth follows from the 1-Lipschitz property of $\lse_\eps$ (Lemma~\ref{lem:lipschitz}) and the non-expansiveness of the reset interpolation. The set of inputs violating (A2) has Lebesgue measure zero.
\end{proposition}

\begin{proof}
By strong induction on $t$. Base case: $V_\eps(0) = v(0)$; spike convergence by Lemma~\ref{lem:sigmoid}. For the inductive step, the update $F_\eps(V) = \tilde{V}(1 - s_\eps) + V_\mathrm{reset} \cdot s_\eps$ with $\tilde{V} = \lse_\eps(V + \log\tau_0, I)$ satisfies: (i)~no spike ($s_\eps \to 0$): $F_\eps \to \max(v(t) + \log\tau_0, I(t)) = v(t+1)$; (ii)~spike ($s_\eps \to 1$): $F_\eps \to V_\mathrm{reset} = 0$, matching LIF reset. Each step adds at most $\eps\log 2$ error (Lemma~\ref{lem:lse_convergence}); the convex reset interpolation does not amplify it. Full details in Appendix~\ref{app:proofs}.
\end{proof}

\begin{corollary}[UltraDLIF Convergence]
\label{cor:ultradlif_convergence}
The analogous result holds for UltraDLIF: as $\eps \to 0^+$, the 3-term $\lse_\eps(V_{i-1}, V_i, V_{i+1}) \to \max(V_{i-1}, V_i, V_{i+1})$ by Lemma~\ref{lem:lse_convergence} with $n{=}3$, and the full UltraDLIF trajectory converges to the max-plus diffusion dynamics (Eq.~\ref{eq:ultradlif_limit}) with error bound $|V_{i,\eps}(t) - V_i(t)| \leq t \cdot \eps \log 3$.
\end{corollary}

\subsection{Gradient Properties}

\begin{proposition}[Bounded Non-Vanishing Gradients]
\label{prop:gradients}
For any $\eps > 0$, consider a single-step spike output $s_\eps = \sigma((\tilde{V}_\eps - \theta)/\eps)$ where $\tilde{V}_\eps$ is the pre-reset voltage. The gradient satisfies $0 < \frac{\partial s_\eps}{\partial \tilde{V}_\eps} \leq \frac{1}{4\eps}$ for all finite $\tilde{V}_\eps$. For weights $W$ with input $x$ at a single layer: $\left|\frac{\partial s_\eps}{\partial W_{ij}}\right| \leq \frac{\|x\|_\infty}{4\eps}$.
\end{proposition}

\begin{proof}
Let $z = (V_\eps - \theta)/\eps$. Then $\frac{\partial s_\eps}{\partial V_\eps} = \frac{\sigma(z)(1-\sigma(z))}{\eps}$. The function $\sigma(1-\sigma)$ achieves maximum $1/4$ at $\sigma = 1/2$, establishing the upper bound. Positivity follows since $\sigma(z) \in (0,1)$ for finite $z$. The weight gradient bound follows by chain rule: $\frac{\partial s_\eps}{\partial W_{ij}} = \frac{\partial s_\eps}{\partial V_\eps} \cdot \frac{\partial V_\eps}{\partial I} \cdot x_j$, with the $\lse$ gradient being softmax with values in $(0,1)$.
\end{proof}

\begin{corollary}[Gradient Scaling]
\label{cor:scaling}
The temperature $\eps$ controls the bias-variance trade-off in gradients:
\begin{itemize}
    \item Larger $\eps$: smaller gradients, smoother optimization landscape
    \item Smaller $\eps$: larger gradients near threshold, better LIF approximation
\end{itemize}
\end{corollary}

\subsection{Forward-Backward Consistency}

\begin{remark}[Gradient Consistency]
\label{thm:consistency}
For any $\eps > 0$, UltraLIF is a composition of smooth operations ($\lse_\eps$, sigmoid, affine maps), so the chain rule applies exactly: the backward pass differentiates the same function computed forward. Surrogate methods use $H(V{-}\theta)$ forward but differentiate a surrogate $g(V{-}\theta)$ backward; \citet{gygax2025elucidating} show this can be interpreted as differentiating a stochastic forward pass with escape noise. UltraLIF avoids this mismatch: $\nabla_W \mathcal{L}(f_\eps(\mathbf{x}; W))$ is an exact gradient of the actual forward computation.
\end{remark}

\subsection{Connection to Tropical Geometry}

\begin{theorem}[Tropical Limit]
\label{thm:tropical}
The map $D_\eps: (\R_{>0}, +, \cdot) \to (\R, \oplus_\eps, +)$ defined by $D_\eps(x) = \eps \log x$, where $a \oplus_\eps b = \lse_\eps(a, b)$, is a semiring homomorphism. As $\eps \to 0^+$, $\oplus_\eps \to \oplus = \max$ (tropical addition), UltraLIF dynamics converge to a piecewise-linear map on $\R_{\max}$, and decision boundaries approach tropical hypersurfaces.
\end{theorem}

\begin{proof}
$D_\eps(x \cdot y) = \eps\log(xy) = D_\eps(x) + D_\eps(y)$ (multiplication $\to$ addition) and $D_\eps(x + y) = \eps\log(x+y) = \lse_\eps(D_\eps(x), D_\eps(y))$ (addition $\to$ soft-max). Taking $\eps \to 0$ yields the tropical semiring by Lemma~\ref{lem:lse_convergence}. The piecewise-linear limit and tropical hypersurface structure follow from \citet{zhang2018tropical}.
\end{proof}

\section{Experiments}
\label{sec:experiments}

\textbf{Setup.} Evaluation spans six benchmarks: static images (MNIST, Fashion-MNIST, CIFAR-10), neuromorphic vision (N-MNIST, DVS-Gesture), and audio (SHD). A single hidden layer with 64 neurons is used across all experiments, with timesteps $T \in \{1, 10, 30\}$. Baselines include LIF, PLIF, AdaLIF, FullPLIF, and DSpike/DSpike+ with surrogate gradients. All four ultradiscretized variants (UltraLIF, UltraPLIF, UltraDLIF, UltraDPLIF) are evaluated. An optional sparsity penalty $\mathcal{L} = \mathcal{L}_\text{CE} + \lambda \cdot \bar{s}$ enables explicit accuracy-efficiency trade-offs. Energy is estimated via the relative synaptic operation (SOP) count $T \cdot \bar{s}$ \citep{lemaire2023analytical}, which is proportional to computational energy since all models share the same architecture (Appendix~\ref{app:setup}).

\begin{table}[t]
\caption{Test accuracy (\%) on CIFAR-10. UltraPLIF (temporal) achieves best at all timesteps.}
\label{tab:cifar10}
\vskip 0.1in
\begin{center}
\begin{small}
\begin{tabular}{lccc}
\toprule
Model & $T{=}1$ & $T{=}10$ & $T{=}30$ \\
\midrule
LIF        & 39.83 & 44.27 & 45.69 \\
PLIF       & 39.83 & 45.06 & 46.15 \\
AdaLIF     & 39.83 & 44.86 & 45.83 \\
FullPLIF   & 39.60 & \underline{45.43} & 46.28 \\
DSpike     & 40.26 & 44.78 & 45.34 \\
DSpike+    & \underline{40.26} & 45.42 & \underline{46.29} \\
\midrule
\multicolumn{4}{l}{\textit{Temporal (LIF ODE)}} \\
UltraLIF   & 40.72 & 45.15 & 45.69 \\
UltraPLIF  & \textbf{43.27} & \textbf{46.19} & \textbf{46.58} \\
\midrule
\multicolumn{4}{l}{\textit{Spatial (Diffusion PDE)}} \\
UltraDLIF  & 43.11 & 45.65 & 45.00 \\
UltraDPLIF & 43.11 & 45.75 & 45.74 \\
\bottomrule
\end{tabular}
\end{small}
\end{center}
\vskip -0.1in
\end{table}

\begin{table}[t]
\caption{Energy efficiency on CIFAR-10. Energy = relative SOP count ($T \cdot \bar{s}$). Sparsity penalty $\lambda$ reduces spike rate with minimal accuracy loss. At $T{=}30$, UltraPLIF with $\lambda{=}0.1$ achieves \textbf{best accuracy} while reducing energy by 50\%.}
\label{tab:efficiency}
\vskip 0.1in
\begin{center}
\begin{small}
\begin{tabular}{lcccc}
\toprule
Model & $T$ & Acc (\%) & Spike & Energy \\
\midrule
LIF        & 1 & 39.83 & 0.404 & 0.40 \\
DSpike+    & 1 & 40.26 & 0.386 & 0.39 \\
UltraPLIF   & 1 & 43.27 & 0.458 & 0.46 \\
UltraPLIF ($\lambda{=}0.1$) & 1 & \textbf{43.60} & \textbf{0.240} & \textbf{0.24} \\
\midrule
PLIF       & 10 & 45.06 & 0.356 & 3.56 \\
UltraDPLIF  & 10 & 45.75 & 0.469 & 4.69 \\
UltraDPLIF ($\lambda{=}0.1$) & 10 & \textbf{45.32} & \textbf{0.338} & \textbf{3.38} \\
\midrule
PLIF       & 30 & 46.15 & 0.377 & 11.30 \\
UltraPLIF  & 30 & 46.58 & 0.500 & 15.01 \\
UltraPLIF ($\lambda{=}0.1$) & 30 & \textbf{46.98} & \textbf{0.248} & \textbf{7.44} \\
\bottomrule
\end{tabular}
\end{small}
\end{center}
\vskip -0.1in
\end{table}

\begin{table}[t]
\caption{Test accuracy (\%) on MNIST. UltraDLIF (spatial) achieves best at $T{=}1$; UltraPLIF (temporal) at $T{=}30$.}
\label{tab:mnist}
\vskip 0.1in
\begin{center}
\begin{small}
\begin{tabular}{lccc}
\toprule
Model & $T{=}1$ & $T{=}10$ & $T{=}30$ \\
\midrule
LIF        & 95.34 & \textbf{97.45} & 97.45 \\
PLIF       & 95.34 & 97.40 & 97.33 \\
AdaLIF     & 95.34 & 97.43 & 97.45 \\
FullPLIF   & 95.27 & 97.33 & 97.31 \\
DSpike     & \underline{95.58} & 97.38 & \underline{97.48} \\
DSpike+    & \underline{95.58} & 97.39 & 97.27 \\
\midrule
\multicolumn{4}{l}{\textit{Temporal (LIF ODE)}} \\
UltraLIF   & 94.37 & 97.14 & 97.46 \\
UltraPLIF  & 95.60 & 97.30 & \textbf{97.55} \\
\midrule
\multicolumn{4}{l}{\textit{Spatial (Diffusion PDE)}} \\
UltraDLIF  & \textbf{95.67} & 97.35 & 97.38 \\
UltraDPLIF & \textbf{95.67} & 97.35 & 97.40 \\
\bottomrule
\end{tabular}
\end{small}
\end{center}
\vskip -0.1in
\end{table}

\begin{table}[t]
\caption{Energy efficiency on MNIST. Energy = relative SOP count ($T \cdot \bar{s}$). UltraDLIF with $\lambda{=}0.1$ reduces spike rate by 40\% (0.446 $\to$ 0.268). At $T{=}10$, UltraDPLIF with $\lambda{=}0.1$ achieves 50\% energy reduction.}
\label{tab:mnist_efficiency}
\vskip 0.1in
\begin{center}
\begin{small}
\begin{tabular}{lcccc}
\toprule
Model & $T$ & Acc (\%) & Spike & Energy \\
\midrule
LIF        & 1 & 95.34 & 0.388 & 0.39 \\
DSpike+    & 1 & 95.58 & 0.403 & 0.40 \\
UltraDLIF   & 1 & 95.67 & 0.446 & 0.45 \\
UltraDLIF ($\lambda{=}0.1$) & 1 & \textbf{95.71} & \textbf{0.268} & \textbf{0.27} \\
\midrule
DSpike+    & 10 & 97.39 & 0.415 & 4.15 \\
UltraDLIF ($\lambda{=}0.01$) & 10 & \textbf{97.56} & 0.448 & 4.48 \\
UltraDPLIF ($\lambda{=}0.1$) & 10 & 97.35 & \textbf{0.237} & \textbf{2.37} \\
\bottomrule
\end{tabular}
\end{small}
\end{center}
\vskip -0.1in
\end{table}

\begin{table}[t]
\caption{Test accuracy (\%) on N-MNIST (neuromorphic). UltraDLIF achieves \textbf{+3.91\%} over baselines at $T{=}1$.}
\label{tab:nmnist}
\vskip 0.1in
\begin{center}
\begin{small}
\begin{tabular}{lccc}
\toprule
Model & $T{=}1$ & $T{=}10$ & $T{=}30$ \\
\midrule
LIF        & 88.54 & 97.48 & 97.29 \\
PLIF       & 88.54 & 97.53 & 97.61 \\
AdaLIF     & 88.54 & 97.38 & 97.30 \\
FullPLIF   & 89.00 & 97.50 & 97.48 \\
DSpike     & 90.23 & 97.39 & 97.59 \\
DSpike+    & \underline{90.23} & \textbf{97.55} & \underline{97.65} \\
\midrule
\multicolumn{4}{l}{\textit{Temporal (LIF ODE)}} \\
UltraLIF   & 90.41 & 96.10 & 95.87 \\
UltraPLIF  & 93.11 & 96.33 & 95.77 \\
\midrule
\multicolumn{4}{l}{\textit{Spatial (Diffusion PDE)}} \\
UltraDLIF  & \textbf{94.14} & 97.38 & 97.46 \\
UltraDPLIF & \textbf{94.14} & 97.38 & \textbf{97.68} \\
\bottomrule
\end{tabular}
\end{small}
\end{center}
\vskip -0.1in
\end{table}

\begin{table}[t]
\caption{Test accuracy (\%) on DVS-Gesture (neuromorphic). UltraPLIF achieves \textbf{+7.96\%} at $T{=}1$.}
\label{tab:dvs}
\vskip 0.1in
\begin{center}
\begin{small}
\begin{tabular}{lccc}
\toprule
Model & $T{=}1$ & $T{=}10$ & $T{=}30$ \\
\midrule
LIF        & 52.27 & 67.05 & \textbf{79.92} \\
PLIF       & \underline{52.27} & \underline{68.94} & 78.79 \\
AdaLIF     & 47.73 & \underline{68.94} & 78.79 \\
FullPLIF   & 47.73 & 68.56 & 77.27 \\
DSpike     & 51.14 & 67.42 & 78.79 \\
DSpike+    & 51.14 & 66.67 & 78.41 \\
\midrule
\multicolumn{4}{l}{\textit{Temporal (LIF ODE)}} \\
UltraLIF   & 58.33 & \textbf{69.32} & 75.00 \\
UltraPLIF  & \textbf{60.23} & 68.94 & 75.76 \\
\midrule
\multicolumn{4}{l}{\textit{Spatial (Diffusion PDE)}} \\
UltraDLIF  & 58.33 & \textbf{69.32} & 78.41 \\
UltraDPLIF & 58.33 & 68.56 & \textbf{79.92} \\
\bottomrule
\end{tabular}
\end{small}
\end{center}
\vskip -0.1in
\end{table}

\begin{table}[t]
\caption{Test accuracy (\%) on SHD (audio). At $T{=}1$, UltraDLIF achieves \textbf{+11.22\%} over the best baseline (FullPLIF). Baselines lead at $T{\geq}10$.}
\label{tab:shd}
\vskip 0.1in
\begin{center}
\begin{small}
\begin{tabular}{lccc}
\toprule
Model & $T{=}1$ & $T{=}10$ & $T{=}30$ \\
\midrule
LIF        & 27.69 & 72.66 & 72.66 \\
PLIF       & 27.69 & 71.69 & 73.19 \\
AdaLIF     & 27.69 & \textbf{74.03} & 71.07 \\
FullPLIF   & \underline{40.02} & 71.69 & 72.48 \\
DSpike     & 38.21 & 72.04 & 70.89 \\
DSpike+    & 38.21 & 72.75 & \textbf{73.85} \\
\midrule
\multicolumn{4}{l}{\textit{Temporal (LIF ODE)}} \\
UltraLIF   & 44.88 & 58.79 & 59.14 \\
UltraPLIF  & 46.91 & 57.73 & 59.45 \\
\midrule
\multicolumn{4}{l}{\textit{Spatial (Diffusion PDE)}} \\
UltraDLIF  & \textbf{51.24} & 67.62 & 71.60 \\
UltraDPLIF & \textbf{51.24} & 68.90 & 67.84 \\
\bottomrule
\end{tabular}
\end{small}
\end{center}
\vskip -0.1in
\end{table}

\begin{table}[t]
\caption{Summary: $T{=}1$ accuracy (\%) across datasets. Best baseline and best ultradiscretized model shown in parentheses. Full results in Tables~\ref{tab:cifar10}--\ref{tab:shd} and Appendix~\ref{app:experiments}.}
\label{tab:t1_all}
\vskip 0.1in
\begin{center}
\begin{small}
\begin{tabular}{lccc}
\toprule
Dataset & Baseline & Best Ultra & $\Delta$ \\
\midrule
MNIST & 95.58 (DSpike) & \textbf{95.67} (DLIF) & +0.09 \\
Fashion & 82.67 (DSpike) & \textbf{83.02} (PLIF) & +0.35 \\
CIFAR-10 & 40.26 (DSpike) & \textbf{43.27} (PLIF) & +3.01 \\
N-MNIST & 90.23 (DSpike) & \textbf{94.14} (DLIF) & +3.91 \\
DVS & 52.27 (PLIF) & \textbf{60.23} (PLIF) & +7.96 \\
SHD & 40.02 (FullPLIF) & \textbf{51.24} (DLIF) & +11.22 \\
\bottomrule
\end{tabular}
\end{small}
\end{center}
\vskip -0.1in
\end{table}

\textbf{Single-Timestep Advantage.} The ultradiscretized models' advantage is most pronounced at $T{=}1$ (Table~\ref{tab:t1_all}), where the model must extract maximum information from a single forward pass. Crucially, \textbf{gains are largest on neuromorphic and temporal datasets}: SHD (+11.22\%), DVS-Gesture (+7.96\%), N-MNIST (+3.91\%), and CIFAR-10 (+3.01\%). On simpler static datasets, gains are smaller but consistent: Fashion-MNIST (+0.35\%), MNIST (+0.09\%). Among the two derivations, UltraDLIF (spatial) wins on N-MNIST, SHD, and MNIST, while UltraPLIF (temporal) wins on CIFAR-10, Fashion-MNIST, and DVS-Gesture (Table~\ref{tab:fashion}), suggesting both variants contribute complementary strengths.

\textbf{Neuromorphic Dataset Performance.} Tables~\ref{tab:nmnist}, \ref{tab:dvs}, and \ref{tab:shd} show that ultradiscretized models dramatically outperform baselines on neuromorphic benchmarks at $T{=}1$. On N-MNIST, UltraDLIF achieves 94.14\% versus 90.23\% for DSpike (+3.91\%). On DVS-Gesture, UltraPLIF (temporal) achieves the best result: 60.23\% versus 52.27\% for PLIF (+7.96\%). These datasets capture asynchronous events from dynamic vision sensors, where the temporal structure is fundamental. The soft max-plus dynamics appear better suited to extract information from sparse, event-driven inputs than hard-thresholding surrogates.

\textbf{Sparsity-Accuracy Trade-off.} Without sparsity penalty, ultradiscretized models tend to have higher spike rates than baselines (e.g., CIFAR-10 $T{=}1$: UltraPLIF 0.458 vs.\ LIF 0.404), as the soft spike $s_\eps \in (0,1)$ contributes nonzero activity even below threshold. The sparsity penalty $\lambda$ addresses this and enables flexible energy-accuracy trade-offs (Tables~\ref{tab:efficiency}, \ref{tab:mnist_efficiency}). On CIFAR-10, $\lambda{=}0.1$ reduces spike rates by 48\% (0.458 $\to$ 0.240) while actually improving accuracy (43.27\% $\to$ 43.60\%). On MNIST, $\lambda{=}0.1$ reduces spike rates by 40\% at $T{=}1$ (0.446 $\to$ 0.268) while maintaining or slightly improving accuracy. Remarkably, at $T{=}30$ on CIFAR-10, UltraPLIF with $\lambda{=}0.1$ achieves \emph{both} best accuracy (46.98\%) and 50\% energy reduction compared to the no-penalty baseline.

\textbf{Timestep Scaling.} Performance gaps narrow at higher $T$, and baselines often overtake: on CIFAR-10, the $T{=}1$ advantage (+3.01\% for UltraPLIF) persists but shrinks at $T{=}10$ and $T{=}30$. On SHD, ultradiscretized models lead dramatically at $T{=}1$ (+11.22\%) but baselines surpass them at $T{\geq}10$. This pattern is consistent: on N-MNIST, UltraDLIF leads by +3.91\% at $T{=}1$ but baselines lead at $T{=}10$. The explanation is that surrogate gradients suffer from forward-backward mismatch, but with sufficient $T$, the averaging effect of spike rate coding masks individual spike errors. At low $T$, each spike carries more information, making the consistency of ultradiscretization more valuable.

\textbf{Learnable Temperature.} The temperature $\eps$ implements automatic curriculum learning, starting soft (large $\eps$) for easy optimization then sharpening (small $\eps$) to approximate discrete spikes. Unlike DSpike's heuristic sharpness parameter, $\eps$ has principled convergence guarantees (Proposition~\ref{prop:convergence}). Ablation studies (Appendix~\ref{app:eps_ablation}) confirm that learned $\eps$ consistently outperforms fixed values, converging to the range 0.66--1.08.

\textbf{Computational Cost.} The $\lse$ operation adds minor overhead ($\sim$5\% wall-clock time) compared to standard LIF, but is fully parallelizable on GPU/TPU.

\section{Discussion}
\label{sec:discussion}

\textbf{Temporal and Spatial Instantiations.} UltraLIF and UltraDLIF arise from applying ultradiscretization to different source equations (LIF ODE vs.\ diffusion PDE), yielding models with distinct computational properties. UltraLIF (2-term LSE) captures temporal membrane dynamics; UltraDLIF (3-term LSE) models lateral spatial diffusion. Both share the same theoretical guarantees from Lemmas and Theorems in Section~\ref{sec:theory}.

\textbf{Relation to Surrogate Gradients.} The soft spike~\eqref{eq:ultralif_spike} superficially resembles sigmoid surrogates. The key distinction is \emph{consistency}: ultradiscretized neurons employ identical soft dynamics in both forward and backward passes, whereas surrogate methods use hard spikes forward with soft gradients backward. This consistency eliminates the gradient mismatch analyzed in \citet{gygax2025elucidating}.

\textbf{Biological Interpretation.} The temperature $\eps$ admits interpretation as neural noise or stochasticity. Biological cortical neurons exhibit highly irregular firing patterns \citep{softky1993highly}; the ultradiscretization framework provides a principled model where $\eps$ quantifies this variability.

\textbf{Tropical Geometry Perspective.} Theorem~\ref{thm:tropical} connects SNN dynamics to tropical algebraic geometry. This opens avenues for applying tropical techniques such as Newton polytopes, Bezout bounds, and tropical Nullstellensatz to analyze SNN expressivity and decision boundaries.

\textbf{Consequence: A New Activation Family.} A direct consequence of the ultradiscretization framework is that UltraLIF subsumes classical activation functions: $\lse_\eps(0, x)$ recovers the softplus (and ReLU as $\eps \to 0$), while $s_\eps$ is a scaled sigmoid. The learnable $\eps$ thus interpolates between hard spiking and smooth regimes. Details and connections to morphological neural networks are in Appendix~\ref{app:discussion}.

\textbf{Limitations.} UltraDLIF and UltraDPLIF produce identical results in our experiments, indicating that the learnable leak $\tau$ does not differentiate the spatial variant; the 3-term LSE dominates the dynamics. Tasks requiring precise spike timing may benefit differently from temporal (UltraLIF) versus spatial (UltraDLIF) variants. The soft spike $s_\eps \in (0,1)$ during training deviates from binary spikes; however, as discussed in Section~\ref{sec:neuron_models}, this is addressed by: (1) using hard thresholding at inference for neuromorphic deployment, (2) the convergence guarantee $s_\eps \to H(V-\theta)$ as $\eps \to 0$, and (3) the learned $\eps$ converging to moderate values (0.66--1.08) that balance differentiability and spike sharpness (Figure~\ref{fig:eps}).

\section{Conclusion}
\label{sec:conclusion}

This paper introduced UltraLIF, a principled framework for differentiable spiking neural networks grounded in ultradiscretization from tropical geometry. Two neuron models, UltraLIF (temporal, from the LIF ODE capturing membrane decay) and UltraDLIF (spatial, from the diffusion PDE modeling gap junction coupling across neuronal populations), use log-sum-exp as a soft relaxation of max-plus dynamics, yielding fully differentiable SNNs without surrogate gradients. Theoretical analysis establishes convergence to classical LIF dynamics, bounded gradients, and forward-backward consistency. Experiments on six benchmarks demonstrate improvements over surrogate gradient baselines, with the largest gains at $T{=}1$ on neuromorphic and audio data (SHD +11.22\%, DVS +7.96\%, N-MNIST +3.91\%). An optional sparsity penalty enables significant energy reduction while maintaining accuracy. The connection to tropical geometry opens new directions for principled analysis of spiking computation.

\section*{Impact Statement}

This paper presents theoretical work whose goal is to advance the field of Machine Learning, specifically in the domain of energy-efficient spiking neural networks. The ultradiscretization framework provides mathematical foundations that could accelerate deployment of neuromorphic computing systems, potentially reducing the energy footprint of machine learning applications. All experiments are validated on existing public benchmark datasets. There are many potential societal consequences of this work, none of which must be specifically highlighted here.

\bibliography{main}
\bibliographystyle{icml2026}

\newpage
\onecolumn
\appendix
\section{Spike Mechanism Comparison}
\label{app:taxonomy}

\begin{figure}[h]
\centering
\includegraphics[width=\textwidth]{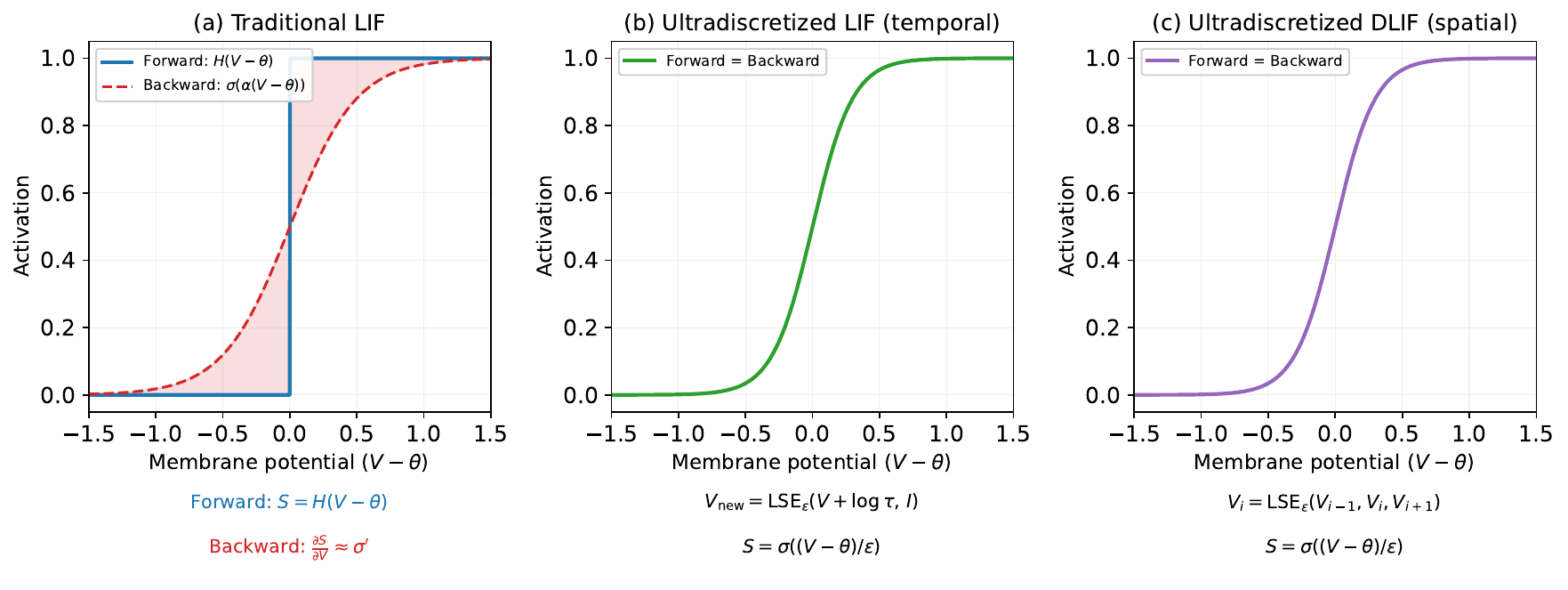}
\caption{Comparison of spike mechanisms. \textbf{(a)} Traditional LIF uses Heaviside $H(V-\theta)$ in the forward pass but a smooth surrogate $\sigma'$ for gradients, creating forward-backward mismatch (shaded region). \textbf{(b)} Ultradiscretized LIF (temporal, 2-term LSE from LIF ODE) and \textbf{(c)} Ultradiscretized DLIF (spatial, 3-term LSE from diffusion PDE) use identical smooth functions in both passes, ensuring gradient consistency. The membrane potential equations below each panel show the distinct derivations.}
\label{fig:taxonomy}
\end{figure}

\section{Additional Discussion}
\label{app:discussion}

\textbf{UltraLIF as a New Activation Family.} The ultradiscretization framework yields a novel activation function family with deep connections to existing architectures. Observe that $\lse_\eps(0, x) = \eps \log(1 + \exp(x/\eps))$ is precisely the \emph{softplus} function, which converges to $\mathrm{ReLU}(x) = \max(0, x)$ as $\eps \to 0$. Thus, the membrane dynamics in UltraLIF generalize ReLU-family activations to the spiking domain. Meanwhile, the spike function $s_\eps = \sigma((V - \theta)/\eps)$ is a shifted, scaled sigmoid, i.e., a soft step function. Together, UltraLIF combines \emph{soft-max} aggregation (ReLU family) with \emph{soft-step} thresholding (sigmoid family), both controlled by the learnable temperature $\eps$. This provides a principled spectrum of activations:
\begin{itemize}[nosep,leftmargin=*]
    \item $\eps \to 0$: Hard max + hard step (classical LIF, non-differentiable)
    \item $\eps \to \infty$: Linear + constant (no nonlinearity)
    \item $\eps$ learned: Optimal sharpness (UltraLIF, fully differentiable)
\end{itemize}
Unlike heuristic surrogate gradients, this activation family has rigorous theoretical grounding in tropical geometry and maintains forward-backward consistency by design.

\textbf{Connection to Morphological Neural Networks.} The spatial max operation in UltraDLIF (Eq.~\ref{eq:ultradlif_limit}) corresponds to \emph{morphological dilation}, a fundamental operation in mathematical morphology. Deep morphological networks \citep{franchi2020deep} show that max pooling is equivalent to dilation with a flat structuring element. This connection suggests that UltraDLIF performs learned morphological operations on neural activation patterns, providing an alternative interpretation grounded in image processing theory.

\section{Proofs}
\label{app:proofs}

\subsection{Proof of Proposition~\ref{prop:convergence}}

The proof proceeds by strong induction on $t$, tracking errors explicitly at each step.

At the base case $t=0$, the voltage error is zero by initialization ($V_\eps(0) = v_0 = v(0)$), and the spike error $|s_\eps(0) - s(0)| \leq e^{-\delta_0/\eps} \to 0$ follows from Lemma~\ref{lem:sigmoid} and the threshold margin assumption~(A2).

For the inductive step, assume $|V_\eps(k) - v(k)| \leq k \cdot \eps \log 2$ holds for all $k \leq t$, and consider the UltraLIF update $F_\eps(V) = \tilde{V}(1 - s_\eps) + V_\mathrm{reset} \cdot s_\eps$ with $\tilde{V} = \lse_\eps(V + \log\tau_0, I)$ at step $t+1$. When no spike occurs at $t$, the inductive hypothesis and Lemma~\ref{lem:sigmoid} imply $s_\eps(t) \to 0$, so the reset interpolation reduces to $V_\eps(t+1) \approx \tilde{V}_\eps(t+1) = \lse_\eps(V_\eps(t) + \log\tau_0, I(t))$. Since $\lse_\eps$ is 1-Lipschitz (Lemma~\ref{lem:lipschitz}), the accumulated error propagates with factor at most 1, and the intrinsic LSE approximation contributes at most $\eps \log 2$ (Lemma~\ref{lem:lse_convergence} with $n=2$), giving $|V_\eps(t{+}1) - v(t{+}1)| \leq |V_\eps(t) - v(t)| + \eps\log 2 \leq (t{+}1)\eps\log 2$. When a spike does occur, $s_\eps(t) \to 1$ drives the reset interpolation toward $V_\mathrm{reset} = 0$, matching the standard LIF reset exactly; the error thus resets to $O(\eps)$, dominated by $(t{+}1)\eps\log 2$.

Crucially, the reset interpolation $\tilde{V}(1-s) + V_\mathrm{reset} \cdot s$ is a convex combination for $s \in (0,1)$, making it non-expansive: perturbations in $\tilde{V}$ are damped by the factor $(1-s) \leq 1$. Combined with the 1-Lipschitz property of $\lse_\eps$, the one-step map never amplifies errors, yielding the linear growth $C_t = t$. Spike convergence at $t+1$ then follows from Lemma~\ref{lem:sigmoid}: $|s_\eps(t{+}1) - s(t{+}1)| \leq e^{-\delta_{t+1}/\eps}$ under Assumption~(A2). \qed

\subsection{Lipschitz Properties}

\begin{lemma}
\label{lem:lipschitz}
For fixed $\eps > 0$, $\lse_\eps: \R^n \to \R$ is 1-Lipschitz in $\|\cdot\|_\infty$, and $\sigma_\eps: \R \to (0,1)$ is $(4\eps)^{-1}$-Lipschitz.
\end{lemma}

\begin{proof}
By the mean value theorem, $|\lse_\eps(\mathbf{x}) - \lse_\eps(\mathbf{y})| \leq \sup_\mathbf{z} |\langle \nabla \lse_\eps(\mathbf{z}), \mathbf{x} - \mathbf{y}\rangle|$. Since $\nabla \lse_\eps = \mathrm{softmax}(\cdot/\eps)$ has $\|\nabla \lse_\eps\|_1 = 1$ (Lemma~\ref{lem:lse_convergence}), H\"older's inequality gives $|\langle \nabla \lse_\eps, \mathbf{x} - \mathbf{y}\rangle| \leq \|\nabla \lse_\eps\|_1 \cdot \|\mathbf{x} - \mathbf{y}\|_\infty = \|\mathbf{x} - \mathbf{y}\|_\infty$. The sigmoid bound follows from $|\sigma_\eps'(x)| \leq 1/(4\eps)$ (Proposition~\ref{prop:gradients}).
\end{proof}

\subsection{Tropical Geometry Analysis}
\label{app:tropical-analysis}

Theorem~\ref{thm:tropical} establishes that UltraLIF dynamics converge to piecewise-linear maps on the max-plus semiring, with decision boundaries approaching tropical hypersurfaces. Three concrete extensions exploiting this connection are developed below, providing explicit expressivity bounds, temporal dynamics analysis, and capacity results.

\subsubsection{Tropical Characterization of Decision Boundaries}
\label{app:tropical-boundaries}

An explicit geometric characterization of UltraLIF decision boundaries in the tropical limit is provided.

\paragraph{Setup.} Consider a single-hidden-layer UltraLIF network with $h$ neurons, $n$ inputs, and $C$ output classes at $T{=}1$ in the tropical limit ($\eps \to 0^+$). Starting from $V^{(0)} = \mathbf{0}$, each hidden neuron $j \in [h]$ computes:
\begin{equation}
\label{eq:app-tropical-neuron}
V_j^{(1)} = \max(\log \tau_0, \; \mathbf{w}_j^\top \mathbf{x})
\end{equation}
where $\mathbf{w}_j \in \mathbb{R}^n$ is the weight vector for neuron $j$.

\begin{definition}[Hyperplane Arrangement]
Each neuron $j$ defines a hyperplane $\mathcal{H}_j := \{\mathbf{x} \in \mathbb{R}^n : \mathbf{w}_j^\top \mathbf{x} = \log \tau_0\}$. The collection $\mathcal{A} = \{\mathcal{H}_1, \ldots, \mathcal{H}_h\}$ partitions $\mathbb{R}^n$ into connected regions where the spike pattern $\mathbf{s} \in \{0,1\}^h$ is constant.
\end{definition}

\begin{proposition}[Hyperplane Arrangement Bound]
\label{prop:app-tropical-arrangement}
A single-hidden-layer UltraLIF partitions $\mathbb{R}^n$ into at most $R(h, n) := \sum_{k=0}^{\min(n,h)} \binom{h}{k}$ linear regions. Within each region, the network output is constant.
\end{proposition}

\begin{proof}
Each neuron defines a half-space $\mathcal{H}_j^+ = \{\mathbf{x} : \mathbf{w}_j^\top \mathbf{x} > \log\tau_0\}$ where $s_j = 1$. The spike pattern is determined by membership in intersections of these half-spaces. The number of regions created by $h$ hyperplanes in $\mathbb{R}^n$ in general position (Theorem~\ref{thm:app-zonotope-expressivity}) is~\citep{zaslavsky1975}:
\[
r_n(\mathcal{A}) = \sum_{k=0}^{n} \binom{h}{k}
\]
When $h < n$, terms with $k > h$ vanish since $\binom{h}{k} = 0$ for $k > h$, yielding the equivalent formula $\sum_{k=0}^{\min(n,h)} \binom{h}{k}$. Within each region, $\mathbf{s}$ is constant, so the output $\hat{y}_c = \sum_{j: s_j=1} W^{\mathrm{out}}_{cj}$ is constant.
\end{proof}

\paragraph{Tropical Hypersurface Structure.} The decision boundary $\mathcal{B}_{ij} = \{\mathbf{x} : \hat{y}_i(\mathbf{x}) = \hat{y}_j(\mathbf{x})\}$ is a union of $(n{-}1)$-faces of the arrangement where class scores are equal. This realizes Theorem~\ref{thm:tropical}'s connection: $\mathcal{B}_{ij}$ is a tropical hypersurface in the sense that it is the locus where two piecewise-linear functions achieve equality.

\subsubsection{Temporal Expressivity Amplification}
\label{app:temporal-expressivity}

\begin{proposition}[Exponential Growth]
\label{prop:app-temporal-expressivity}
A single-hidden-layer UltraLIF unrolled for $T$ timesteps partitions $\mathbb{R}^n$ into at most $R(h,n)^T$ regions.
\end{proposition}

\begin{proof}
Let $\phi^{(t)}: \mathbb{R}^n \to \mathbb{R}^n$ denote the map from input to hidden layer membrane potentials at timestep $t$. Each $\phi^{(t)}$ is piecewise-linear with at most $R(h,n)$ regions by Proposition~\ref{prop:app-tropical-arrangement}. The $T$-step computation is the composition $\phi^{(T)} \circ \cdots \circ \phi^{(1)}$.

By induction on $T$, the composition has at most $R(h,n)^T$ linear regions. Base case ($T=1$): immediate from Proposition~\ref{prop:app-tropical-arrangement}. Inductive step: assume the claim holds for $T-1$. Then $\phi^{(T-1)} \circ \cdots \circ \phi^{(1)}$ has at most $R(h,n)^{T-1}$ regions. Composing with $\phi^{(T)}$ (which has $R(h,n)$ regions) yields at most $R(h,n)^{T-1} \cdot R(h,n) = R(h,n)^T$ regions, since each region of the $(T-1)$-step map can be subdivided by the $R(h,n)$ hyperplanes of $\phi^{(T)}$.
\end{proof}

\paragraph{T=1 vs T$\geq$10 Analysis.} At $T{=}1$, expressivity is $R(h,n) \approx 10^{19}$ for typical settings, far exceeding dataset size. The advantage comes from gradient quality: UltraLIF's forward-backward consistency avoids spurious local minima induced by surrogate gradient mismatch.

At $T{\geq}10$, expressivity grows to $R(h,n)^{10}$. Output averaging $\hat{y} = \frac{1}{T}\sum_t W^{\mathrm{out}}\mathbf{s}^{(t)}$ smooths individual spike errors. Gradient mismatch becomes less critical as errors cancel across timesteps, explaining why baselines recover (SHD: UltraLIF +11.22\% at $T{=}1$ but $-2.16\%$ at $T{=}30$).

\subsubsection{Zonotope Volume and Expressivity}
\label{app:zonotope}

\begin{definition}[Zonotope]
The zonotope generated by $\mathbf{w}_1, \ldots, \mathbf{w}_h \in \mathbb{R}^n$ is $\mathcal{Z}(\mathbf{w}_1, \ldots, \mathbf{w}_h) := \{\sum_{i=1}^h \lambda_i \mathbf{w}_i : \lambda_i \in [0,1]\}$.
\end{definition}

The volume $\mathrm{vol}_n(\mathcal{Z})$ quantifies geometric diversity of weight directions. When $\mathrm{vol}_n(\mathcal{Z}) = 0$, weights are linearly dependent and the arrangement degenerates.

\begin{theorem}[Volume-Expressivity Connection]
\label{thm:app-zonotope-expressivity}
Let $W = [\mathbf{w}_1, \ldots, \mathbf{w}_h]^\top \in \mathbb{R}^{h \times n}$ be the weight matrix. The hyperplane arrangement achieves the maximal region count $R(h,n) = \sum_k \binom{h}{k}$ if and only if the hyperplanes are in general position. General position holds if and only if:
\begin{enumerate}
\item For any $n{+}1$ weight vectors $\mathbf{w}_{i_1}, \ldots, \mathbf{w}_{i_{n+1}}$, no point $\mathbf{x}$ satisfies all $n{+}1$ hyperplane equations simultaneously.
\item Any subset of $n$ weight vectors $\{\mathbf{w}_{i_1}, \ldots, \mathbf{w}_{i_n}\}$ with $n \leq \min(h, \dim(\mathbb{R}^n))$ has $\det(\mathbf{w}_{i_1}, \ldots, \mathbf{w}_{i_n}) \neq 0$ when $n = \dim(\mathbb{R}^n)$.
\end{enumerate}
Moreover, if $h \geq n$ and the weight matrix $W$ has rank $n$, then $\mathrm{vol}_n(\mathcal{Z}(\mathbf{w}_1, \ldots, \mathbf{w}_h)) > 0$ implies the arrangement is non-degenerate with at least $2^n$ regions.
\end{theorem}

\begin{proof}
The first statement follows from Zaslavsky's characterization of general position~\citep{zaslavsky1975}: the arrangement achieves the maximal count when no $n{+}1$ hyperplanes meet at a point and all intersections are transverse. Condition (1) ensures no common intersection of $n{+}1$ hyperplanes. Condition (2) ensures transversality: any $n$ hyperplanes intersect at a unique point (when $n$ weight vectors are linearly independent) rather than a higher-dimensional face.

For the volume statement, suppose $h \geq n$ and $\mathrm{rank}(W) = n$. Then there exist $n$ linearly independent weight vectors, say $\mathbf{w}_{i_1}, \ldots, \mathbf{w}_{i_n}$. The zonotope contains the parallelepiped spanned by these vectors:
\[
\mathcal{P} = \{\sum_{j=1}^n \lambda_j \mathbf{w}_{i_j} : \lambda_j \in [0,1]\}
\]
The volume satisfies $\mathrm{vol}_n(\mathcal{Z}) \geq \mathrm{vol}_n(\mathcal{P}) = |\det(\mathbf{w}_{i_1}, \ldots, \mathbf{w}_{i_n})| > 0$ by linear independence.

These $n$ linearly independent hyperplanes partition $\mathbb{R}^n$ into at least $2^n$ regions (the $2^n$ orthants when hyperplanes pass through the origin, or their translated analogues). Thus $R \geq 2^n$ when $\mathrm{vol}_n(\mathcal{Z}) > 0$.

Conversely, if $\mathrm{vol}_n(\mathcal{Z}) = 0$, the weight vectors lie in a proper subspace of dimension $< n$, yielding a degenerate arrangement with $R < 2^n$ (at most linear in $h$).
\end{proof}

\begin{corollary}[Capacity Lower Bound]
\label{cor:app-capacity-lower-bound}
For a single-hidden-layer UltraLIF with $h \geq n$ neurons and $\mathrm{rank}(W) = n$:
\[
\text{Expressivity} \geq 2^n \quad \text{if } \mathrm{vol}_n(\mathcal{Z}(\mathbf{w}_1, \ldots, \mathbf{w}_h)) > 0
\]
\end{corollary}

\paragraph{Implications.}

\textit{Initialization.} Standard random initialization (Kaiming, Xavier) samples from isotropic Gaussians, yielding near-orthogonal weights with high probability in high dimensions. This ensures $\mathrm{vol}_n(\mathcal{Z}) > 0$ and non-degenerate arrangements. Explicit orthogonalization (QR decomposition) maximizes volume for fixed norms.

\textit{Regularization.} A volume-regularized loss $\mathcal{L} = \mathcal{L}_{\text{CE}} - \alpha \log \mathrm{vol}_n(\mathcal{Z})$ encourages diverse weight directions, preventing rank collapse.

\textit{Pruning.} If $\mathbf{w}_j \approx c \cdot \mathbf{w}_k$, removing neuron $j$ reduces $\mathrm{vol}_n(\mathcal{Z})$ negligibly, providing a principled pruning criterion based on geometric redundancy.

\section{Experimental Setup}
\label{app:setup}

\subsection{Datasets and Preprocessing}

Evaluation is conducted on six benchmarks spanning static images, neuromorphic vision, and audio:

\textbf{Static Image Datasets.}
\begin{itemize}
    \item \textbf{MNIST}: $28{\times}28$ grayscale handwritten digits, 10 classes. 60,000 train / 10,000 test samples. Normalization: mean 0.1307, std 0.3081.
    \item \textbf{Fashion-MNIST}: $28{\times}28$ grayscale fashion items, 10 classes. 60,000 train / 10,000 test samples. Normalization: mean 0.2860, std 0.3530.
    \item \textbf{CIFAR-10}: $32{\times}32$ RGB natural images, 10 classes. 50,000 train / 10,000 test samples. Normalization per channel: mean (0.4914, 0.4822, 0.4465), std (0.247, 0.243, 0.262). Training augmentation: random crop (32$\times$32 with 4-pixel padding), random horizontal flip.
\end{itemize}

\textbf{Neuromorphic Datasets.}
\begin{itemize}
    \item \textbf{N-MNIST}: Neuromorphic MNIST captured with DVS camera. 60,000 train / 10,000 test samples. Input dimension: $2{\times}34{\times}34$ (ON/OFF polarity channels).
    \item \textbf{DVS-Gesture}: 11 hand gesture classes recorded with DVS128 camera. 1,176 train / 288 test samples. Input dimension: $2{\times}128{\times}128$.
\end{itemize}

\textbf{Audio Dataset.}
\begin{itemize}
    \item \textbf{SHD}: Spiking Heidelberg Digits, 20 spoken digit classes. 8,156 train / 2,264 test samples. Input dimension: 700 frequency channels.
\end{itemize}

\textbf{Temporal Encoding.} For static datasets, rate coding converts pixel intensities to spike trains:
\begin{equation}
P(\text{spike at } t) = \text{gain} \cdot x_{\text{pixel}}, \quad \text{gain} = 0.5
\end{equation}
where $x_{\text{pixel}} \in [0, 1]$ is the normalized pixel value. For neuromorphic datasets, events are binned into $T$ temporal frames using the Tonic library's \texttt{ToFrame} transform with \texttt{n\_time\_bins}${}=T$.

\subsection{Model Hyperparameters}

\begin{table}[h]
\caption{Hyperparameters for all neuron models.}
\label{tab:hyperparams}
\vskip 0.1in
\begin{center}
\begin{small}
\begin{tabular}{lcc}
\toprule
Parameter & Value & Description \\
\midrule
\multicolumn{3}{l}{\textit{Common parameters}} \\
$\theta$ (threshold) & 0.5 & Spike threshold \\
$\tau_0$ (leak) & 0.9 & Membrane time constant \\
$V_{\text{reset}}$ & 0.0 & Reset potential \\
\midrule
\multicolumn{3}{l}{\textit{Surrogate gradient (baselines)}} \\
$\beta$ (sharpness) & 10.0 & Sigmoid surrogate steepness \\
\midrule
\multicolumn{3}{l}{\textit{AdaLIF specific}} \\
$\beta_{\text{adapt}}$ & 0.1 & Threshold adaptation strength \\
$\tau_{\text{adapt}}$ & 0.9 & Adaptation decay constant \\
\midrule
\multicolumn{3}{l}{\textit{DSpike specific}} \\
$b_0$ (init) & 4.0 & Initial sharpness parameter \\
\midrule
\multicolumn{3}{l}{\textit{Ultradiscretized models}} \\
$\eps_0$ (init) & 1.0 & Initial temperature \\
$\eps$ range & [0.1, 20.0] & Clamped during training \\
\bottomrule
\end{tabular}
\end{small}
\end{center}
\vskip -0.1in
\end{table}

\subsection{Architecture}

Single hidden layer with 64 neurons. Input-to-hidden and hidden-to-output are fully connected layers. Timesteps $T \in \{1, 10, 30\}$ to evaluate performance across temporal regimes. Output is computed as mean spike rate over time:
\begin{equation}
\hat{y} = \frac{1}{T}\sum_{t=1}^T W_{\text{out}} \cdot s^{(t)}
\end{equation}

\subsection{Baseline Model Definitions}

All baselines use surrogate gradients: hard spike $s = H(v - \theta)$ in the forward pass, smooth gradient $\partial s / \partial v = \sigma'((v-\theta)\beta)$ in the backward pass.

\textbf{LIF} (Leaky Integrate-and-Fire):
\begin{align}
v^{(t+1)} &= \tau v^{(t)} + I^{(t)} \\
s^{(t+1)} &= H(v^{(t+1)} - \theta), \quad v \leftarrow v(1-s)
\end{align}

\textbf{PLIF} (Parametric LIF, \citealp{fang2021incorporating}):
\begin{align}
v^{(t+1)} &= \tau v^{(t)} + I^{(t)}, \quad \tau = \sigma(\tau_{\text{param}}) \text{ learnable}
\end{align}

\textbf{AdaLIF} (Adaptive LIF, \citealp{bellec2018long}):
\begin{align}
v^{(t+1)} &= \tau v^{(t)} + I^{(t)} \\
\theta^{(t+1)} &= \theta_0 + \beta_{\text{adapt}} \cdot b^{(t)} \\
b^{(t+1)} &= \tau_{\text{adapt}} \cdot b^{(t)} + (1 - \tau_{\text{adapt}}) \cdot s^{(t)}
\end{align}
where $b$ is the adaptation variable that increases after spikes.

\textbf{FullPLIF} (Fully Parametric LIF):
\begin{align}
v^{(t+1)} &= \tau v^{(t)} + I^{(t)} \\
\tau &= \sigma(\tau_{\text{param}}), \quad \theta = \sigma(\theta_{\text{param}})
\end{align}
Both $\tau$ and $\theta$ are learnable, constrained to $(0, 1)$ via sigmoid.

\textbf{DSpike} (\citealp{li2021differentiable}):
\begin{align}
v^{(t+1)} &= \tau v^{(t)} + I^{(t)} \\
s^{(t+1)} &= \frac{\tanh(k(v_{\text{norm}} - 0.5)) + \tanh(k/2)}{2\tanh(k/2)}
\end{align}
where $v_{\text{norm}} = v / (2\theta)$ normalizes membrane potential to $[0, 1]$, and $k$ is a learnable sharpness parameter (initialized to 4.0).

\textbf{DSpike+}: DSpike with learnable $\tau = \sigma(\tau_{\text{param}})$.

\subsection{Proposed Methods}

Ultradiscretized variants use soft spike in both forward and backward passes (no surrogate gradients):

\textbf{UltraLIF} (Temporal, 2-term LSE from LIF ODE):
\begin{align}
V^{(t+1)} &= \lse_\eps(V^{(t)} + \log\tau_0, \, I^{(t)}) \\
s_\eps^{(t+1)} &= \sigma((V^{(t+1)} - \theta) / \eps)
\end{align}

\textbf{UltraDLIF} (Spatial, 3-term LSE from diffusion PDE):
\begin{align}
V_i^{(t+1)} &= \lse_\eps(V_{i-1}^{(t)}, V_i^{(t)}, V_{i+1}^{(t)}) + I_i^{(t)} \\
s_{i,\eps}^{(t+1)} &= \sigma((V_i^{(t+1)} - \theta) / \eps)
\end{align}

UltraPLIF and UltraDPLIF add learnable $\tau = \sigma(\tau_{\text{param}})$.

\subsection{Sparsity Penalty}

To encourage energy efficiency, ultradiscretized variants support an optional sparsity penalty:
\begin{equation}
\mathcal{L} = \mathcal{L}_\text{CE} + \lambda \cdot \bar{s}
\end{equation}
where $\bar{s}$ is the mean spike rate and $\lambda \in \{0, 0.01, 0.1\}$. This enables explicit control over the accuracy-efficiency trade-off.

\subsection{Training Details}

\begin{table}[h]
\caption{Training configuration.}
\label{tab:training}
\vskip 0.1in
\begin{center}
\begin{small}
\begin{tabular}{lc}
\toprule
Parameter & Value \\
\midrule
Optimizer & Adam \\
Learning rate & $10^{-3}$ \\
Batch size & 128 \\
Epochs & 100 \\
LR scheduler & Cosine annealing \\
Weight init & PyTorch default (Kaiming) \\
Seed & 42 \\
\bottomrule
\end{tabular}
\end{small}
\end{center}
\vskip -0.1in
\end{table}

\textbf{Loss Function.} Cross-entropy on mean spike rates:
\begin{equation}
\mathcal{L}_\text{CE} = -\sum_c y_c \log(\text{softmax}(\hat{y})_c)
\end{equation}

\textbf{Energy Estimation.} SNN energy consumption is estimated using the standard synaptic operation (SOP) framework \citep{lemaire2023analytical}. In SNNs, spike-driven computation replaces multiply-accumulate (MAC) operations with accumulate-only (AC) operations, since pre-synaptic activations are binary:
\begin{align}
E_\text{SNN} &= T \cdot \bar{s} \cdot N_\text{syn} \cdot E_\text{AC} \\
E_\text{ANN} &= N_\text{syn} \cdot E_\text{MAC}
\end{align}
where $T$ is the number of timesteps, $\bar{s}$ is the mean spike rate, $N_\text{syn}$ is the number of synaptic connections, $E_\text{AC} \approx 0.9\,\text{pJ}$, and $E_\text{MAC} \approx 4.6\,\text{pJ}$ at 45nm technology \citep{horowitz20141}. Since all models in this work share the same architecture (and thus the same $N_\text{syn}$), the energy column in results tables reports the \textbf{relative SOP count} $T \cdot \bar{s}$, which is proportional to $E_\text{SNN}$ up to a constant factor. This enables direct energy comparison across models. Note that this estimate focuses on computational energy and does not account for memory access or data movement overhead, which can be significant in practice \citep{yan2024reconsidering}.

\subsection{Hardware and Compute}

All experiments were conducted on NVIDIA T4 GPUs (16GB VRAM) using PyTorch 2.0+ with \texttt{torch.compile} for acceleration. Training time per model:
\begin{itemize}
    \item MNIST/Fashion: $\sim$3 minutes (100 epochs)
    \item CIFAR-10: $\sim$8 minutes (100 epochs)
    \item N-MNIST: $\sim$15 minutes (100 epochs)
    \item DVS-Gesture: $\sim$20 minutes (100 epochs)
    \item SHD: $\sim$10 minutes (100 epochs)
\end{itemize}
Total compute for all experiments: approximately 50 GPU-hours across three T4 VMs.

\textbf{Note on neuromorphic hardware.} All experiments in this work were conducted on conventional GPUs. Deployment on dedicated neuromorphic hardware (e.g., Intel Loihi 2, IBM TrueNorth, SpiNNaker 2, BrainScaleS-2) has not yet been evaluated. Since the ultradiscretized spike function converges to hard thresholding as $\eps \to 0$ (Proposition~\ref{prop:convergence}), the trained models can in principle be mapped to neuromorphic substrates by quantizing the soft spike to binary at inference time. Benchmarking latency, energy consumption, and accuracy on neuromorphic platforms is an important direction for future work.

\section{Additional Experiments}
\label{app:experiments}

\begin{figure}[h]
\centering
\includegraphics[width=\linewidth]{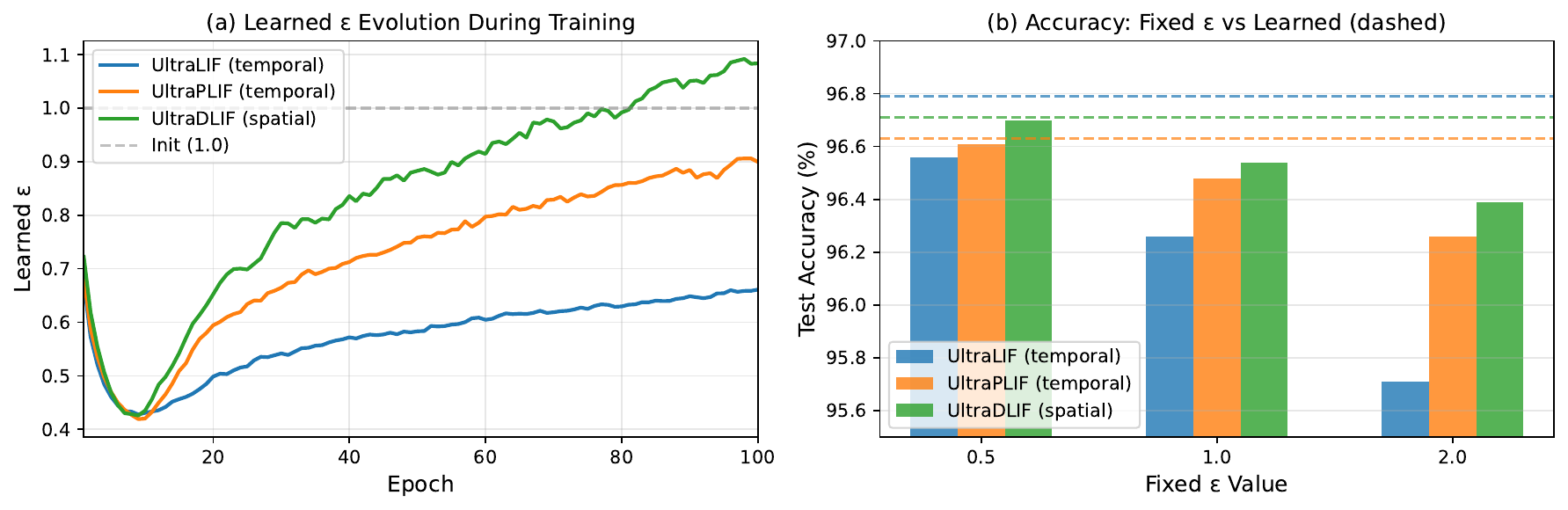}
\caption{Epsilon ablation on MNIST ($T{=}1$, 100 epochs). (a) Learned $\eps$ exhibits a characteristic U-shaped trajectory: initial drop from 1.0 to $\sim$0.42 (sharpening phase), followed by recovery to model-specific optima (0.66--1.08). This suggests the network first learns sharp discrimination, then softens for generalization. (b) Learned $\eps$ (dashed lines) consistently matches or exceeds all fixed values across models, validating the benefit of learnable temperature.}
\label{fig:eps}
\end{figure}

\subsection{Ablation Study: Learnable Temperature $\eps$}
\label{app:eps_ablation}

A key design choice in ultradiscretized neurons is whether to fix the temperature parameter $\eps$ or learn it during training. This ablation study on MNIST at $T{=}1$ compares fixed values $\eps \in \{0.5, 1.0, 2.0\}$ against learned $\eps$ (initialized to 1.0) across all four ultradiscretized variants.

\begin{table}[h]
\caption{Ablation: Effect of learnable $\eps$ on accuracy (\%). Learned $\eps$ consistently achieves best or tied-best accuracy across all models.}
\label{tab:eps_ablation_acc}
\vskip 0.1in
\begin{center}
\begin{small}
\begin{tabular}{lcccc}
\toprule
Model & $\eps{=}0.5$ & $\eps{=}1.0$ & $\eps{=}2.0$ & Learned \\
\midrule
\multicolumn{5}{l}{\textit{Temporal (LIF ODE)}} \\
UltraLIF   & 96.56 & 96.26 & 95.71 & \textbf{96.79} \\
UltraPLIF  & 96.61 & 96.48 & 96.26 & \textbf{96.63} \\
\midrule
\multicolumn{5}{l}{\textit{Spatial (Diffusion PDE)}} \\
UltraDLIF  & 96.70 & 96.54 & 96.39 & \textbf{96.71} \\
UltraDPLIF & 96.70 & 96.54 & 96.39 & \textbf{96.71} \\
\bottomrule
\end{tabular}
\end{small}
\end{center}
\vskip -0.1in
\end{table}

\begin{table}[h]
\caption{Ablation: Effect of learnable $\eps$ on spike rate. Learned $\eps$ achieves lowest spike rates for spatial models.}
\label{tab:eps_ablation_spike}
\vskip 0.1in
\begin{center}
\begin{small}
\begin{tabular}{lcccc}
\toprule
Model & $\eps{=}0.5$ & $\eps{=}1.0$ & $\eps{=}2.0$ & Learned \\
\midrule
\multicolumn{5}{l}{\textit{Temporal (LIF ODE)}} \\
UltraLIF   & 0.488 & 0.579 & 0.650 & 0.500 \\
UltraPLIF  & 0.427 & 0.459 & 0.528 & 0.443 \\
\midrule
\multicolumn{5}{l}{\textit{Spatial (Diffusion PDE)}} \\
UltraDLIF  & 0.417 & 0.412 & 0.423 & \textbf{0.393} \\
UltraDPLIF & 0.417 & 0.412 & 0.423 & \textbf{0.393} \\
\bottomrule
\end{tabular}
\end{small}
\end{center}
\vskip -0.1in
\end{table}

\textbf{Key findings:} (1) Learned $\eps$ provides consistent accuracy gains across all models, though margins are small on MNIST. (2) Smaller fixed $\eps$ (0.5) outperforms larger values (2.0), consistent with Lemma~\ref{lem:lse_convergence}, as tighter approximation to the max yields better LIF emulation. (3) Learned $\eps$ converges to the range 0.66--1.08 (Table~\ref{tab:eps_final}), automatically finding optimal soft-to-hard trade-offs. (4) Spatial models achieve lowest spike rates with learned $\eps$, suggesting the network learns to balance accuracy and efficiency.

\begin{table}[h]
\caption{Final learned $\eps$ values after training (initialized at 1.0).}
\label{tab:eps_final}
\vskip 0.1in
\begin{center}
\begin{small}
\begin{tabular}{lc}
\toprule
Model & Final $\eps$ \\
\midrule
UltraLIF (temporal)  & 0.661 \\
UltraPLIF (temporal) & 0.900 \\
UltraDLIF (spatial)  & 1.084 \\
UltraDPLIF (spatial) & 1.084 \\
\bottomrule
\end{tabular}
\end{small}
\end{center}
\vskip -0.1in
\end{table}

Table~\ref{tab:fashion} reports full results on Fashion-MNIST. UltraPLIF (temporal) achieves the best accuracy at $T{=}1$, while baselines lead at higher timesteps.

\begin{table}[h]
\caption{Test accuracy (\%) on Fashion-MNIST. UltraPLIF (temporal) achieves best at $T{=}1$. Baselines lead at $T{\geq}10$.}
\label{tab:fashion}
\vskip 0.1in
\begin{center}
\begin{small}
\begin{tabular}{lccc}
\toprule
Model & $T{=}1$ & $T{=}10$ & $T{=}30$ \\
\midrule
LIF        & 82.45 & 86.06 & 86.70 \\
PLIF       & 82.45 & 85.99 & 86.48 \\
AdaLIF     & 82.45 & 86.12 & 86.59 \\
FullPLIF   & 82.18 & \textbf{86.26} & 86.33 \\
DSpike     & \underline{82.67} & 86.24 & 86.42 \\
DSpike+    & \underline{82.67} & 86.03 & \textbf{86.76} \\
\midrule
\multicolumn{4}{l}{\textit{Temporal (LIF ODE)}} \\
UltraLIF   & 81.79 & 85.76 & 86.65 \\
UltraPLIF  & \textbf{83.02} & 86.03 & 86.59 \\
\midrule
\multicolumn{4}{l}{\textit{Spatial (Diffusion PDE)}} \\
UltraDLIF  & 82.79 & 85.88 & 86.01 \\
UltraDPLIF & 82.79 & 85.69 & 85.92 \\
\bottomrule
\end{tabular}
\end{small}
\end{center}
\vskip -0.1in
\end{table}

\subsection{Full Sparsity Results}
\label{app:sparsity}

Tables~\ref{tab:sparsity_mnist}--\ref{tab:sparsity_shd} present full sparsity results for all ultradiscretized models across sparsity penalty values $\lambda \in \{0, 0.01, 0.1\}$ and timesteps $T \in \{1, 10, 30\}$.

\begin{table}[h]
\caption{Sparsity results on MNIST. Accuracy (\%), spike rate, and relative SOP count ($T \cdot \bar{s}$).}
\label{tab:sparsity_mnist}
\vskip 0.1in
\begin{center}
\begin{small}
\begin{tabular}{llccc}
\toprule
Model & $\lambda$ & Acc & Spike & Energy \\
\midrule
\multicolumn{5}{l}{\textit{$T{=}1$}} \\
UltraLIF   & 0    & 94.37 & 0.620 & 0.62 \\
UltraLIF   & 0.01 & 94.55 & 0.608 & 0.61 \\
UltraLIF   & 0.1  & 94.37 & 0.537 & 0.54 \\
UltraPLIF  & 0    & 95.60 & 0.468 & 0.47 \\
UltraPLIF  & 0.01 & 95.50 & 0.453 & 0.45 \\
UltraPLIF  & 0.1  & \textbf{95.81} & 0.289 & 0.29 \\
UltraDLIF  & 0    & 95.67 & 0.446 & 0.45 \\
UltraDLIF  & 0.01 & 95.62 & 0.425 & 0.43 \\
UltraDLIF  & 0.1  & 95.71 & 0.268 & 0.27 \\
UltraDPLIF & 0    & 95.67 & 0.446 & 0.45 \\
UltraDPLIF & 0.01 & 95.62 & 0.425 & 0.43 \\
UltraDPLIF & 0.1  & 95.71 & \textbf{0.268} & \textbf{0.27} \\
\midrule
\multicolumn{5}{l}{\textit{$T{=}10$}} \\
UltraLIF   & 0    & 97.14 & 0.519 & 5.19 \\
UltraLIF   & 0.01 & 97.15 & 0.498 & 4.98 \\
UltraLIF   & 0.1  & 97.23 & 0.328 & 3.28 \\
UltraPLIF  & 0    & 97.30 & 0.492 & 4.92 \\
UltraPLIF  & 0.01 & 97.28 & 0.466 & 4.66 \\
UltraPLIF  & 0.1  & 97.37 & 0.242 & 2.42 \\
UltraDLIF  & 0    & 97.35 & 0.479 & 4.79 \\
UltraDLIF  & 0.01 & \textbf{97.56} & 0.448 & 4.48 \\
UltraDLIF  & 0.1  & 97.19 & 0.254 & 2.54 \\
UltraDPLIF & 0    & 97.35 & 0.476 & 4.76 \\
UltraDPLIF & 0.01 & 97.37 & 0.444 & 4.44 \\
UltraDPLIF & 0.1  & 97.35 & \textbf{0.237} & \textbf{2.37} \\
\midrule
\multicolumn{5}{l}{\textit{$T{=}30$}} \\
UltraLIF   & 0    & 97.46 & 0.502 & 15.07 \\
UltraLIF   & 0.01 & 97.41 & 0.475 & 14.26 \\
UltraLIF   & 0.1  & 97.51 & 0.270 & 8.09 \\
UltraPLIF  & 0    & \textbf{97.55} & 0.492 & 14.76 \\
UltraPLIF  & 0.01 & 97.52 & 0.464 & 13.91 \\
UltraPLIF  & 0.1  & 97.53 & 0.229 & 6.87 \\
UltraDLIF  & 0    & 97.38 & 0.469 & 14.07 \\
UltraDLIF  & 0.01 & 97.52 & 0.427 & 12.81 \\
UltraDLIF  & 0.1  & 97.11 & \textbf{0.208} & \textbf{6.24} \\
UltraDPLIF & 0    & 97.40 & 0.481 & 14.43 \\
UltraDPLIF & 0.01 & 97.34 & 0.439 & 13.18 \\
UltraDPLIF & 0.1  & 96.99 & 0.209 & 6.27 \\
\bottomrule
\end{tabular}
\end{small}
\end{center}
\vskip -0.1in
\end{table}

\begin{table}[h]
\caption{Sparsity results on Fashion-MNIST. Accuracy (\%), spike rate, and relative SOP count ($T \cdot \bar{s}$).}
\label{tab:sparsity_fashion}
\vskip 0.1in
\begin{center}
\begin{small}
\begin{tabular}{llccc}
\toprule
Model & $\lambda$ & Acc & Spike & Energy \\
\midrule
\multicolumn{5}{l}{\textit{$T{=}1$}} \\
UltraLIF   & 0    & 81.79 & 0.652 & 0.65 \\
UltraLIF   & 0.01 & 82.06 & 0.629 & 0.63 \\
UltraLIF   & 0.1  & 81.62 & 0.530 & 0.53 \\
UltraPLIF  & 0    & 83.02 & 0.472 & 0.47 \\
UltraPLIF  & 0.01 & 82.81 & 0.451 & 0.45 \\
UltraPLIF  & 0.1  & \textbf{83.26} & 0.279 & 0.28 \\
UltraDLIF  & 0    & 82.79 & 0.429 & 0.43 \\
UltraDLIF  & 0.01 & 83.01 & 0.411 & 0.41 \\
UltraDLIF  & 0.1  & 83.05 & 0.267 & 0.27 \\
UltraDPLIF & 0    & 82.79 & 0.429 & 0.43 \\
UltraDPLIF & 0.01 & 83.01 & 0.411 & 0.41 \\
UltraDPLIF & 0.1  & 83.05 & \textbf{0.267} & \textbf{0.27} \\
\midrule
\multicolumn{5}{l}{\textit{$T{=}10$}} \\
UltraLIF   & 0    & 85.76 & 0.522 & 5.22 \\
UltraLIF   & 0.01 & 86.07 & 0.510 & 5.10 \\
UltraLIF   & 0.1  & 85.98 & 0.334 & 3.34 \\
UltraPLIF  & 0    & 86.03 & 0.493 & 4.93 \\
UltraPLIF  & 0.01 & 85.93 & 0.457 & 4.57 \\
UltraPLIF  & 0.1  & \textbf{86.11} & \textbf{0.273} & \textbf{2.73} \\
UltraDLIF  & 0    & 85.88 & 0.456 & 4.56 \\
UltraDLIF  & 0.01 & 85.63 & 0.442 & 4.42 \\
UltraDLIF  & 0.1  & 85.74 & 0.292 & 2.92 \\
UltraDPLIF & 0    & 85.69 & 0.456 & 4.56 \\
UltraDPLIF & 0.01 & 85.86 & 0.427 & 4.27 \\
UltraDPLIF & 0.1  & 85.74 & 0.278 & 2.78 \\
\midrule
\multicolumn{5}{l}{\textit{$T{=}30$}} \\
UltraLIF   & 0    & 86.65 & 0.507 & 15.22 \\
UltraLIF   & 0.01 & 86.46 & 0.485 & 14.56 \\
UltraLIF   & 0.1  & 86.36 & 0.287 & 8.60 \\
UltraPLIF  & 0    & 86.59 & 0.480 & 14.41 \\
UltraPLIF  & 0.01 & 86.49 & 0.453 & 13.59 \\
UltraPLIF  & 0.1  & \textbf{86.72} & \textbf{0.271} & \textbf{8.13} \\
UltraDLIF  & 0    & 86.01 & 0.467 & 14.01 \\
UltraDLIF  & 0.01 & 85.81 & 0.429 & 12.88 \\
UltraDLIF  & 0.1  & 85.94 & 0.274 & 8.23 \\
UltraDPLIF & 0    & 85.92 & 0.463 & 13.88 \\
UltraDPLIF & 0.01 & 85.80 & 0.458 & 13.75 \\
UltraDPLIF & 0.1  & 85.84 & 0.276 & 8.27 \\
\bottomrule
\end{tabular}
\end{small}
\end{center}
\vskip -0.1in
\end{table}

\begin{table}[h]
\caption{Sparsity results on CIFAR-10. Accuracy (\%), spike rate, and relative SOP count ($T \cdot \bar{s}$).}
\label{tab:sparsity_cifar10}
\vskip 0.1in
\begin{center}
\begin{small}
\begin{tabular}{llccc}
\toprule
Model & $\lambda$ & Acc & Spike & Energy \\
\midrule
\multicolumn{5}{l}{\textit{$T{=}1$}} \\
UltraLIF   & 0    & 40.72 & 0.706 & 0.71 \\
UltraLIF   & 0.01 & 40.58 & 0.662 & 0.66 \\
UltraLIF   & 0.1  & 39.81 & 0.459 & 0.46 \\
UltraPLIF  & 0    & 43.27 & 0.458 & 0.46 \\
UltraPLIF  & 0.01 & 43.22 & 0.444 & 0.44 \\
UltraPLIF  & 0.1  & \textbf{43.60} & \textbf{0.240} & \textbf{0.24} \\
UltraDLIF  & 0    & 43.11 & 0.481 & 0.48 \\
UltraDLIF  & 0.01 & 43.13 & 0.465 & 0.47 \\
UltraDLIF  & 0.1  & 43.04 & 0.337 & 0.34 \\
UltraDPLIF & 0    & 43.11 & 0.481 & 0.48 \\
UltraDPLIF & 0.01 & 43.13 & 0.465 & 0.47 \\
UltraDPLIF & 0.1  & 43.04 & 0.337 & 0.34 \\
\midrule
\multicolumn{5}{l}{\textit{$T{=}10$}} \\
UltraLIF   & 0    & 45.15 & 0.504 & 5.04 \\
UltraLIF   & 0.01 & 44.72 & 0.494 & 4.94 \\
UltraLIF   & 0.1  & 45.18 & 0.311 & 3.11 \\
UltraPLIF  & 0    & 46.19 & 0.494 & 4.94 \\
UltraPLIF  & 0.01 & 46.06 & 0.457 & 4.57 \\
UltraPLIF  & 0.1  & 46.13 & \textbf{0.241} & \textbf{2.41} \\
UltraDLIF  & 0    & 45.65 & 0.471 & 4.71 \\
UltraDLIF  & 0.01 & 45.58 & 0.452 & 4.52 \\
UltraDLIF  & 0.1  & 45.39 & 0.334 & 3.34 \\
UltraDPLIF & 0    & 45.75 & 0.469 & 4.69 \\
UltraDPLIF & 0.01 & \textbf{46.26} & 0.452 & 4.52 \\
UltraDPLIF & 0.1  & 45.32 & 0.338 & 3.38 \\
\midrule
\multicolumn{5}{l}{\textit{$T{=}30$}} \\
UltraLIF   & 0    & 45.69 & 0.480 & 14.39 \\
UltraLIF   & 0.01 & 45.84 & 0.466 & 13.99 \\
UltraLIF   & 0.1  & 45.92 & 0.285 & 8.54 \\
UltraPLIF  & 0    & 46.58 & 0.500 & 15.01 \\
UltraPLIF  & 0.01 & 46.31 & 0.485 & 14.56 \\
UltraPLIF  & 0.1  & \textbf{46.98} & \textbf{0.248} & \textbf{7.44} \\
UltraDLIF  & 0    & 45.00 & 0.491 & 14.73 \\
UltraDLIF  & 0.01 & 45.32 & 0.446 & 13.38 \\
UltraDLIF  & 0.1  & 45.41 & 0.322 & 9.65 \\
UltraDPLIF & 0    & 45.74 & 0.496 & 14.89 \\
UltraDPLIF & 0.01 & 46.09 & 0.450 & 13.49 \\
UltraDPLIF & 0.1  & 45.79 & 0.331 & 9.94 \\
\bottomrule
\end{tabular}
\end{small}
\end{center}
\vskip -0.1in
\end{table}

\begin{table}[h]
\caption{Sparsity results on N-MNIST. Accuracy (\%), spike rate, and relative SOP count ($T \cdot \bar{s}$).}
\label{tab:sparsity_nmnist}
\vskip 0.1in
\begin{center}
\begin{small}
\begin{tabular}{llccc}
\toprule
Model & $\lambda$ & Acc & Spike & Energy \\
\midrule
\multicolumn{5}{l}{\textit{$T{=}1$}} \\
UltraLIF   & 0    & 90.41 & 0.579 & 0.58 \\
UltraLIF   & 0.01 & 90.37 & 0.597 & 0.60 \\
UltraLIF   & 0.1  & 90.74 & 0.546 & 0.55 \\
UltraPLIF  & 0    & 93.11 & 0.396 & 0.40 \\
UltraPLIF  & 0.01 & 92.73 & 0.409 & 0.41 \\
UltraPLIF  & 0.1  & 93.28 & 0.318 & 0.32 \\
UltraDLIF  & 0    & \textbf{94.14} & 0.506 & 0.51 \\
UltraDLIF  & 0.01 & 93.97 & 0.492 & 0.49 \\
UltraDLIF  & 0.1  & 93.52 & 0.291 & 0.29 \\
UltraDPLIF & 0    & \textbf{94.14} & 0.506 & 0.51 \\
UltraDPLIF & 0.01 & 93.97 & 0.492 & 0.49 \\
UltraDPLIF & 0.1  & 93.52 & \textbf{0.291} & \textbf{0.29} \\
\midrule
\multicolumn{5}{l}{\textit{$T{=}10$}} \\
UltraLIF   & 0    & 96.10 & 0.447 & 4.47 \\
UltraLIF   & 0.01 & 96.22 & 0.415 & 4.15 \\
UltraLIF   & 0.1  & 96.39 & 0.222 & 2.22 \\
UltraPLIF  & 0    & 96.33 & 0.445 & 4.45 \\
UltraPLIF  & 0.01 & 96.34 & 0.358 & 3.58 \\
UltraPLIF  & 0.1  & 96.60 & \textbf{0.127} & \textbf{1.27} \\
UltraDLIF  & 0    & 97.38 & 0.460 & 4.60 \\
UltraDLIF  & 0.01 & 97.37 & 0.335 & 3.35 \\
UltraDLIF  & 0.1  & 97.00 & 0.155 & 1.55 \\
UltraDPLIF & 0    & 97.38 & 0.463 & 4.63 \\
UltraDPLIF & 0.01 & \textbf{97.40} & 0.306 & 3.06 \\
UltraDPLIF & 0.1  & 96.93 & 0.144 & 1.44 \\
\midrule
\multicolumn{5}{l}{\textit{$T{=}30$}} \\
UltraLIF   & 0    & 95.87 & 0.404 & 12.13 \\
UltraLIF   & 0.01 & 95.77 & 0.391 & 11.72 \\
UltraLIF   & 0.1  & 96.30 & 0.267 & 8.02 \\
UltraPLIF  & 0    & 95.77 & 0.463 & 13.90 \\
UltraPLIF  & 0.01 & 95.91 & 0.412 & 12.36 \\
UltraPLIF  & 0.1  & 96.48 & 0.240 & 7.21 \\
UltraDLIF  & 0    & 97.46 & 0.429 & 12.86 \\
UltraDLIF  & 0.01 & 97.34 & 0.301 & 9.04 \\
UltraDLIF  & 0.1  & 97.28 & 0.140 & 4.20 \\
UltraDPLIF & 0    & \textbf{97.68} & 0.430 & 12.89 \\
UltraDPLIF & 0.01 & 97.52 & 0.275 & 8.26 \\
UltraDPLIF & 0.1  & 97.33 & \textbf{0.125} & \textbf{3.74} \\
\bottomrule
\end{tabular}
\end{small}
\end{center}
\vskip -0.1in
\end{table}

\begin{table}[h]
\caption{Sparsity results on DVS-Gesture. Accuracy (\%), spike rate, and relative SOP count ($T \cdot \bar{s}$).}
\label{tab:sparsity_dvs}
\vskip 0.1in
\begin{center}
\begin{small}
\begin{tabular}{llccc}
\toprule
Model & $\lambda$ & Acc & Spike & Energy \\
\midrule
\multicolumn{5}{l}{\textit{$T{=}1$}} \\
UltraLIF   & 0    & 58.33 & 0.726 & 0.73 \\
UltraLIF   & 0.01 & 57.58 & 0.746 & 0.75 \\
UltraLIF   & 0.1  & 57.95 & 0.690 & 0.69 \\
UltraPLIF  & 0    & \textbf{60.23} & 0.619 & 0.62 \\
UltraPLIF  & 0.01 & 57.20 & 0.610 & 0.61 \\
UltraPLIF  & 0.1  & 55.68 & \textbf{0.601} & \textbf{0.60} \\
UltraDLIF  & 0    & 58.33 & 0.774 & 0.77 \\
UltraDLIF  & 0.01 & 56.44 & 0.833 & 0.83 \\
UltraDLIF  & 0.1  & 58.71 & 0.790 & 0.79 \\
UltraDPLIF & 0    & 58.33 & 0.774 & 0.77 \\
UltraDPLIF & 0.01 & 56.44 & 0.833 & 0.83 \\
UltraDPLIF & 0.1  & 58.71 & 0.790 & 0.79 \\
\midrule
\multicolumn{5}{l}{\textit{$T{=}10$}} \\
UltraLIF   & 0    & 69.32 & 0.707 & 7.07 \\
UltraLIF   & 0.01 & 68.56 & 0.709 & 7.09 \\
UltraLIF   & 0.1  & 67.80 & 0.702 & 7.02 \\
UltraPLIF  & 0    & 68.94 & 0.559 & 5.59 \\
UltraPLIF  & 0.01 & 69.70 & 0.567 & 5.67 \\
UltraPLIF  & 0.1  & 70.83 & \textbf{0.543} & \textbf{5.43} \\
UltraDLIF  & 0    & 69.32 & 0.611 & 6.11 \\
UltraDLIF  & 0.01 & 71.97 & 0.601 & 6.01 \\
UltraDLIF  & 0.1  & 70.45 & 0.543 & 5.43 \\
UltraDPLIF & 0    & 68.56 & 0.615 & 6.15 \\
UltraDPLIF & 0.01 & 71.97 & 0.614 & 6.14 \\
UltraDPLIF & 0.1  & \textbf{73.11} & 0.578 & 5.78 \\
\midrule
\multicolumn{5}{l}{\textit{$T{=}30$}} \\
UltraLIF   & 0    & 75.00 & 0.719 & 21.58 \\
UltraLIF   & 0.01 & 75.38 & 0.719 & 21.56 \\
UltraLIF   & 0.1  & 75.00 & 0.707 & 21.20 \\
UltraPLIF  & 0    & 75.76 & 0.593 & 17.79 \\
UltraPLIF  & 0.01 & 75.76 & 0.588 & 17.63 \\
UltraPLIF  & 0.1  & 76.14 & 0.559 & 16.77 \\
UltraDLIF  & 0    & 78.41 & 0.560 & 16.79 \\
UltraDLIF  & 0.01 & 78.41 & 0.552 & 16.56 \\
UltraDLIF  & 0.1  & \textbf{81.06} & \textbf{0.501} & \textbf{15.04} \\
UltraDPLIF & 0    & 79.92 & 0.570 & 17.09 \\
UltraDPLIF & 0.01 & 77.65 & 0.556 & 16.68 \\
UltraDPLIF & 0.1  & 79.55 & 0.521 & 15.62 \\
\bottomrule
\end{tabular}
\end{small}
\end{center}
\vskip -0.1in
\end{table}

\begin{table}[h]
\caption{Sparsity results on SHD. Accuracy (\%), spike rate, and relative SOP count ($T \cdot \bar{s}$).}
\label{tab:sparsity_shd}
\vskip 0.1in
\begin{center}
\begin{small}
\begin{tabular}{llccc}
\toprule
Model & $\lambda$ & Acc & Spike & Energy \\
\midrule
\multicolumn{5}{l}{\textit{$T{=}1$}} \\
UltraLIF   & 0    & 44.88 & 0.551 & 0.55 \\
UltraLIF   & 0.01 & 45.76 & 0.542 & 0.54 \\
UltraLIF   & 0.1  & 46.86 & 0.507 & 0.51 \\
UltraPLIF  & 0    & 46.91 & 0.390 & 0.39 \\
UltraPLIF  & 0.01 & 47.61 & 0.394 & 0.39 \\
UltraPLIF  & 0.1  & 48.32 & \textbf{0.344} & \textbf{0.34} \\
UltraDLIF  & 0    & 51.24 & 0.686 & 0.69 \\
UltraDLIF  & 0.01 & 50.62 & 0.629 & 0.63 \\
UltraDLIF  & 0.1  & \textbf{51.33} & 0.565 & 0.56 \\
UltraDPLIF & 0    & 51.24 & 0.686 & 0.69 \\
UltraDPLIF & 0.01 & 50.62 & 0.629 & 0.63 \\
UltraDPLIF & 0.1  & \textbf{51.33} & 0.565 & 0.56 \\
\midrule
\multicolumn{5}{l}{\textit{$T{=}10$}} \\
UltraLIF   & 0    & 58.79 & 0.472 & 4.72 \\
UltraLIF   & 0.01 & 57.99 & 0.462 & 4.62 \\
UltraLIF   & 0.1  & 60.51 & 0.397 & 3.97 \\
UltraPLIF  & 0    & 57.73 & 0.420 & 4.20 \\
UltraPLIF  & 0.01 & 58.48 & 0.411 & 4.11 \\
UltraPLIF  & 0.1  & 58.79 & \textbf{0.348} & \textbf{3.48} \\
UltraDLIF  & 0    & 67.62 & 0.416 & 4.16 \\
UltraDLIF  & 0.01 & 69.52 & 0.442 & 4.42 \\
UltraDLIF  & 0.1  & 69.92 & 0.367 & 3.67 \\
UltraDPLIF & 0    & 68.90 & 0.461 & 4.61 \\
UltraDPLIF & 0.01 & 65.77 & 0.496 & 4.96 \\
UltraDPLIF & 0.1  & \textbf{70.27} & 0.399 & 3.99 \\
\midrule
\multicolumn{5}{l}{\textit{$T{=}30$}} \\
UltraLIF   & 0    & 59.14 & 0.458 & 13.73 \\
UltraLIF   & 0.01 & 59.41 & 0.451 & 13.52 \\
UltraLIF   & 0.1  & 61.00 & 0.404 & 12.11 \\
UltraPLIF  & 0    & 59.45 & 0.449 & 13.48 \\
UltraPLIF  & 0.01 & 60.16 & 0.445 & 13.35 \\
UltraPLIF  & 0.1  & 60.60 & \textbf{0.385} & \textbf{11.54} \\
UltraDLIF  & 0    & 71.60 & 0.459 & 13.76 \\
UltraDLIF  & 0.01 & 70.23 & 0.458 & 13.73 \\
UltraDLIF  & 0.1  & \textbf{73.19} & 0.386 & 11.57 \\
UltraDPLIF & 0    & 67.84 & 0.454 & 13.62 \\
UltraDPLIF & 0.01 & 67.71 & 0.449 & 13.48 \\
UltraDPLIF & 0.1  & 69.88 & 0.404 & 12.13 \\
\bottomrule
\end{tabular}
\end{small}
\end{center}
\vskip -0.1in
\end{table}

\textbf{Key observations across datasets:} (1) Moderate sparsity ($\lambda{=}0.1$) consistently reduces spike rates by 40--50\% with minimal accuracy loss, and in several cases (MNIST $T{=}1$, Fashion $T{=}1$, CIFAR-10 $T{=}30$) actually \emph{improves} accuracy, suggesting that sparsity acts as a regularizer. (2) The sparsity-accuracy trade-off is most favorable on temporal models (UltraPLIF), which achieve the lowest energy at competitive accuracy.

\end{document}